\documentclass[10pt,twocolumn,letterpaper]{article}

\usepackage{cvpr}              

\graphicspath{{pic/}}

\usepackage{caption}
\usepackage[accsupp]{axessibility}
\usepackage{graphicx}%
\usepackage{multirow}%
\usepackage{amsmath,amssymb,amsfonts}%
\usepackage{amsthm}%
\usepackage{mathrsfs}%
\usepackage[title]{appendix}%
\usepackage{xcolor}
\usepackage{textcomp}%
\usepackage{manyfoot}%
\usepackage{caption}
\usepackage{diagbox}
\usepackage[export]{adjustbox}
\usepackage{tcolorbox} 
\usepackage{listings}%
\usepackage{epsfig}
\usepackage{placeins}
\usepackage{makecell}
\usepackage{float}
\usepackage{booktabs}%
\usepackage{bm}
\usepackage{setspace}
\usepackage[accsupp]{axessibility}
\usepackage{subcaption} 
\usepackage[ruled,linesnumbered]{algorithm2e}

\usepackage{graphicx}
\usepackage{amsmath}
\usepackage{amssymb}
\usepackage{booktabs}

\newtheorem{proposition}{Proposition}[section]
\newcommand{\mytips}[1][]{%
    \begin{tcolorbox}[
        colframe=black,
        boxrule=1pt,
        arc=6pt,
        left=3pt,
        right=3pt,
        top=3pt,
        bottom=3pt,
        width=\linewidth,
        #1 
    ]
}

\newcommand{\diff}{\mathop{}\!\mathrm{d}}

\definecolor{darkgreen}{rgb}{0.0, 0.5, 0.0}

\usepackage{url}            
\usepackage{booktabs}       
\usepackage{amsfonts}       
\usepackage{nicefrac}       
\usepackage{microtype}      
\usepackage{xcolor}         
\usepackage{tabularx}
\usepackage{subcaption}
\usepackage{multicol}
\usepackage{multirow}
\usepackage{tabularray}
\usepackage{amsmath}
\usepackage{amssymb}
\usepackage{siunitx}
\usepackage{colortbl}
\usepackage{graphicx}
\usepackage{lipsum}
\usepackage{afterpage}
\usepackage{placeins}
\usepackage{pgfplots,gincltex}
\usepgfplotslibrary{groupplots}
\pgfplotsset{compat=1.6}
\usepackage{tikz}
\usetikzlibrary{arrows.meta}
\usetikzlibrary{intersections}   
\usepgfplotslibrary{fillbetween} 
\usepackage[accsupp]{axessibility}  

\usepackage{balance}

\definecolor{cvprblue}{rgb}{0.21,0.49,0.74}
\definecolor{pltblue}{RGB}{174, 199, 232}
\definecolor{pltorange}{RGB}{255, 229, 204}
\definecolor{pltgreen}{RGB}{204, 229, 204}
\definecolor{pltred}{RGB}{229, 204, 204}
\definecolor{pltpurple}{RGB}{239, 218, 230}

\definecolor{tabblue}{HTML}{1f77b4}
\definecolor{taborange}{HTML}{ff7f0e}
\definecolor{tabgreen}{HTML}{2ca02c}
\definecolor{tabred}{HTML}{d62728}
\definecolor{tabpurple}{HTML}{9467bd}

\definecolor{cblue}{RGB}{173, 201, 233}
\definecolor{clblue}{RGB}{222, 234, 246}
\definecolor{corange}{RGB}{255, 152, 67}
\definecolor{lorgange}{RGB}{255, 221, 149}
\definecolor{myred}{RGB}{174,66,38}

\newcolumntype{C}{>{\centering\arraybackslash}X}

\newcommand{\cc}[1]{\cellcolor{clblue!50}{#1}}

\definecolor{cvprblue}{rgb}{0.21,0.49,0.74}
\usepackage[pagebackref,breaklinks,colorlinks,allcolors=cvprblue]{hyperref}

\def\paperID{xxxx} 
\def\confName{CVPR}
\def\confYear{2026}

\title{Decoupled Residual Denoising Diffusion Models for 
Unified and Data Efficient Image-to-Image Translation}

\author{%
Ziyue Lin$^{1}$\footnotemark[1], Jiahe Hou$^{2}$\footnotemark[1], Hongyu Xia$^{1}$\footnotemark[1], Xinrui Xie$^{3}$, Feifei Wang$^{1}$, Yuyin Zhou$^{4}$, Wei Wang$^2$\\ Jiawei Liu$^{2}$\footnotemark[2], Liangqiong Qu$^{1}$\footnotemark[2]\\
$^1$The University of Hong Kong  \hspace*{0.5em}  $^2$Shenyang Institute of Automation, Chinese Academy of Sciences\\
$^3$The Chinese University of Hong Kong  \hspace*{0.5em}  $^4$University of California, Santa Cruz \\
\texttt{\{ziyue\_lin,xiahyu\}@connect.hku.hk, liujiawei@sia.cn, liangqqu@hku.hk}
}

\begin{document}
\maketitle

\begin{abstract}
We propose Decoupled Residual Denoising Diffusion models (DRDD) for unified and data-efficient image-to-image (I2I) translation. While diffusion models have advanced I2I translation in terms of quality and diversity, we uncover a previously under-explored property in diffusion models. Crucially, beyond its conventional role of manifold lifting (i.e., moving data off low-dimensional manifolds), injecting Gaussian noise facilitates domain harmonization by implicitly aligning feature distributions across domains, a property particularly advantageous for unified I2I translation. However, existing diffusion models prematurely erode this harmonization effect, as noise and residuals are simultaneously removed in a single coupled diffusion process. To address this, DRDD decouples the diffusion process into two sequential and independent diffusion stages: (1) a stochastic noise diffusion for domain harmonization and manifold lifting, and (2) a deterministic residual diffusion that learns the core semantic mapping entirely within the fixed-noise domain.  This decoupling preserves harmonization and manifold lifting effects throughout the transformation, substantially simplifying the learning of unified mappings across diverse tasks and domains. Notably, the noise diffusion stage is trained exclusively on abundant, unpaired target-domain images, greatly improving data efficiency. Comprehensive theoretical and empirical analysis demonstrates that DRDD is broadly compatible with mainstream diffusion models and consistently delivers robust, unified I2I translation, even under limited paired data. Our code is available at \href{https://github.com/HKU-HealthAI/DRDD}{https://github.com/HKU-HealthAI/DRDD}.
\end{abstract}

\renewcommand{\thefootnote}{\fnsymbol{footnote}}
\footnotetext[1]{Equally contributed.  $^{\dagger}$Corresponding author.}
\renewcommand{\thefootnote}{\arabic{footnote}}


\section{Introduction}
Image-to-image (I2I) translation aims to map an input image from a source domain to a target domain, a fundamental task in computer vision with wide-ranging applications   \citep{saharia2022palette, pang2021imagetoimagetranslationmethodsapplications}, including image restoration \citep{zheng2024diffuir, li2024foundir}, super-resolution \citep{SR3}, image style translation \citep{choi2021ilvrconditioningmethoddenoising,meng2022sdeditguidedimagesynthesis}, among others \citep{lugmayr2022repaintinpaintingusingdenoising}. Early approaches to I2I translation have been dominated by Generative Adversarial Networks (GANs) \citep{GAN, gatedgan, cyclegan}. However, GANs suffer from training instability and limited mode coverage \citep{mescheder2018trainingmethodsgansactually}. More recently, denoising diffusion models have set new benchmarks for output quality and diversity by learning a reversible process of iteratively adding and removing noise \citep{NEURIPS2020_DDPM, dhariwal2021diffusion, lipman2022flow}. Early diffusion-based I2I methods, such as SR3 \citep{SR3} and WeatherDiff \citep{weatherdiff}, typically initiate the reverse process from pure Gaussian noise, using the input image solely as a conditioning signal. To more stably preserve structural information of the input and reduce inference uncertainty, 
advanced approaches, e.g., RDDM \citep{liu2024rddm} and I2SB \citep{I2SB},  no longer initialize from pure noise; instead, they start from noise-carrying input image \citep{luo2023IRSDE, indi}. Even with variant starting points for reverse sampling, they share an underlying principle: image-to-image translation is achieved through a single, \textbf{coupled} reverse process, where noise and residuals are simultaneously removed at each diffusion step. 
\vspace{-0.5mm}
\begin{figure*}[t]
    \centering 
    \includegraphics[width=0.9\textwidth]{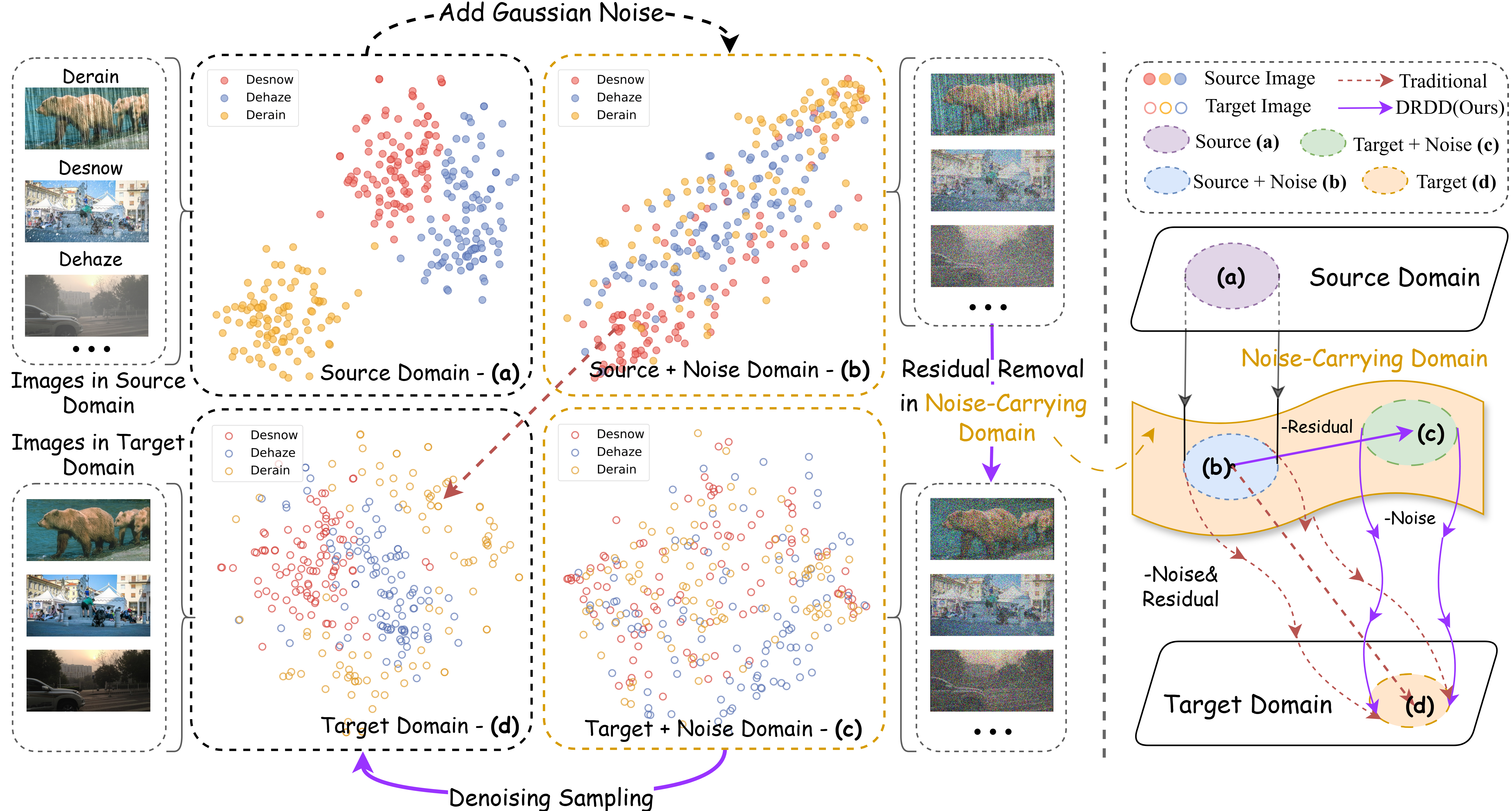}
   \caption{\textbf{Left: Domain gap reduction via noise introduction.} t-SNE plot of feature representations across three I2I translation tasks. The original source domain (a) shows a significant domain gap, complicating unified I2I translation. Introducing noise (Source+Noise domain, b) reduces the
domain gap, as shown by noticeably closer feature representation. Here, \textit{noise-carrying domains} refer to domains with noise added to original images (e.g., Source+Noise and Target+Noise domains) as noise-carrying domains, excluding pure noise. 
\textbf{Right: Reverse process comparison.} During the inference stage, traditional coupled diffusion models perform domain shifting from the noise-carrying input to the target by simultaneously removing the residual (the difference between the source and target) and the noise [(b) $\rightarrow$ (d)]. Different from them, DRDD first performs residual removal exclusively within the noise-carrying domain, and then performs a denoising process to transform noisy-carrying target into clean target [(b) $\rightarrow$ (c) $\rightarrow$ (d)].}
    \label{fig:intro}
\end{figure*}

Despite these promising advances, applying these \textbf{coupled} diffusion models to unified I2I translation, where one model must handle multiple distinct tasks and domains, remains challenging. The primary difficulty stems from the substantial domain gaps between different I2I translation tasks, as well as the challenge of collecting large-scale paired source-to-target images that adequately cover this diversity. Here, we re-examine the diffusion paradigm to enhance its suitability for data-efficient and unified I2I translation. A key insight from our investigation concerns the role of injected Gaussian noise in diffusion models. Beyond its conventional functions of moving data off low-dimensional manifolds and enriching training signals for score estimation \citep{song2019generative, lai2025principlesdiffusionmodels}, we theoretically and empirically demonstrate that injecting a certain level of Gaussian noise can act as a \textbf{``domain harmonizer"}, \textbf{pulling feature representations from disparate domains closer together} (see Fig.~\ref{fig:intro}(a), (b) and Proposition \ref{proposition:noise_reduct_KL}). This emergent property, which is particularly beneficial for unified I2I translation tasks, reveals an under-explored utility of noise in diffusion processes. Nevertheless, in prevailing coupled diffusion models \citep{liu2024rddm,luo2023IRSDE,I2SB,indi}, the reverse process progressively removes the injected noise, thereby \textbf{eroding this harmonization benefit before the source-to-target mapping is completed}. This ultimately undermines the model's effectiveness in a unified, data-efficient setting. \\
\indent To fully unleash the role of injected Gaussian noise, we propose \textbf{D}ecoupled \textbf{R}esidual \textbf{D}enoising \textbf{D}iffusion models for I2I Translation (denoted as DRDD). As shown in Fig.~\ref{fig:intro} and Fig.~\ref{fig:main_fig}, DRDD fundamentally decouples the conventional single forward diffusion process into two independent and sequential diffusion processes: (1) a stochastic noise diffusion stage first injects Gaussian perturbations for domain harmonization and manifold lifting, projecting the target domain into a harmonized, noise-carrying domain; (2) this is followed by a deterministic residual diffusion stage that learns the target-to-source mapping within this fixed-noise domain. The reverse process is symmetrically decoupled. It first performs residual removal within a noise-carrying domain (with a predefined and fixed noise level) to achieve the core source-to-target transformation, thereby preserving the domain harmonization and manifold lifting effects. This is followed by a denoising stage for fidelity refinement. This entire reverse process is clearly visualized in \textit{Fig.~\ref{fig:intro}(b) $\rightarrow$ Fig.~\ref{fig:intro}(c) $\rightarrow$ Fig.~\ref{fig:intro}(d)} and Fig.~\ref{fig:main_fig}. \\
\indent This novel decoupling mechanism offers two key advantages. First, by performing the core source-to-target mapping before any noise is removed, the domain harmonization and manifold lifting effects of the initial noise persist, thereby substantially simplifying the learning of a unified image mapping. Second, the design inherently enhances data efficiency, as the denoising stage is trained exclusively on target-domain images. This enables DRDD to leverage abundant unpaired data to boost final fidelity. As a result,  DRDD establishes a new paradigm for unified and data-efficient I2I transformation, delivering robust performance across diverse tasks, even with limited paired data. Notably, both theoretical analysis and empirical results confirm that our residual and noise decoupling idea is compatible with multiple popular diffusion paradigms, including DDPM \citep{NEURIPS2020_DDPM}, DDIM \citep{song2022ddim}, and SDE-based diffusion models \citep{song2021scorebased}. Our core contributions are listed as follows:   
\begin{itemize}
    \item We uncover and formalize a novel role of Gaussian noise as a ``domain harmonizer" in diffusion models. We theoretically and empirically demonstrate that controlled noise injection can effectively bridge the representation gap between disparate domains, a property particularly advantageous for unified I2I translation.

    \item We propose DRDD, a novel diffusion method that decouples standard diffusion into sequential noise diffusion and residual diffusion stages. This decoupling strategically separates domain harmonization from semantic mapping, ensuring the harmonization effect persists throughout the core source-to-target transformation by performing residual removal entirely within the noisy domain.

    \item Our decoupled design establishes a new paradigm for data-efficient and unified I2I translation. It not only simplifies the learning of a unified mapping across tasks but also enables the denoising stage to be trained exclusively on abundant, unpaired target-domain images, achieving robust performance with limited paired data. We further validate the broad compatibility of our framework with mainstream diffusion paradigms.

\end{itemize}


\section{Related Works}
Image-to-image translation (I2I) aims to transfer images from a source domain to a target domain while preserving the content representations, with wide applications in many computer vision tasks \citep{Zamir2021Restormer, liang2021swinir,SRGan,SR3, saharia2022palette,choi2021ilvrconditioningmethoddenoising,meng2022sdeditguidedimagesynthesis, xie2025arra, zheng2025fedvlmbench}. Diffusion models have shown their impressive performance in I2I tasks. SR3 \citep{SR3} first applies diffusion models in I2I, focusing on super-resolution and pioneers the usage of the input image as a condition for sampling from pure noise to a clear image. Subsequent works \citep{lugmayr2022repaintinpaintingusingdenoising, choi2021ilvrconditioningmethoddenoising, saharia2022palette} extend diffusion-based I2I works to image inpainting, style transformation and colorization. Meanwhile, mainstream diffusion-based I2I approaches like RDDM \citep{liu2024rddm} and others \citep{luo2023IRSDE, I2SB, indi} enhance performance by initializing the reverse process with a noisy input to better preserve input information and reduce uncertainty. This strategy drives the model to perform denoising and the required domain translation (e.g., residual removal) simultaneously within a single, coupled reverse process.

\section{Method}
\label{headings}



\subsection{Motivation}
The widespread popularity of diffusion models originates with denoising diffusion probabilistic models (DDPMs) \citep{NEURIPS2020_DDPM}. Since then, denoising and diffusion have been tightly coupled across a range of generative and I2I translation tasks. This close coupling has sometimes fostered the impression that ``\textit{the denoising {\bf network} itself is responsible for producing a clean target image containing semantic information.}'' 


We revisit this impression for two reasons. (1) If we examine only the objective function of DDPM \citep{NEURIPS2020_DDPM}, the network learns solely denoising capabilities, which has no direct connection to generating a clear target image. 
We argue that {\bf the generative semantic capability of diffusion models} does not stem from the denoising network itself, nor even from the denoising process\footnote{RDDM \citep{liu2024rddm} reveals that, whether reverse sampling begins from pure noise (e.g., vanilla DDPM \citep{NEURIPS2020_DDPM} and SR3 \citep{SR3}) or from noise-carrying input image (e.g., I2SB \cite{I2SB} and IR-SDE \citep{luo2023IRSDE} with additional diffusion terms), they all involve residual diffusion and noise diffusion. 
Our experiments in Fig.~\ref{fig:main_fig} also show that the denoising process itself only learns to remove noise and cannot generate semantic information corresponding to clear target images.}, but rather {\bf a) from the mutual representation between the predicted noise $\epsilon_{\theta}$ and the predicted target image $I_0^{\theta}$ (i.e., $I_0^{\theta}=f(\epsilon_{\theta})$, see Eq.9 in DDIM \citep{song2022ddim} and Eq.16 in RDDM \citep{liu2024rddm}); b) from the sampling formula derived from the mutual representation}. 
(2)  When diffusion models were extended from image generation to I2I translation, researchers observed that predicting residual \citep{whang2022deblurring,liu2023flow}, target images \citep{bansal2022cold, indi} or its linear transformation terms \cite{I2SB} often yielded better results than predicting noise. But noise injecting is still used empirically in I2I translations because its addition has been observed to improve performance \citep{indi}. 
These observations collectively suggest that the role of noise is nuanced. However, the community lacks a clear understanding of noise's role in I2I  translation. 

In this paper, we discover an additional role for noise in I2I translation tasks, i.e., domain harmonization that injects noise can reduce the distance between feature representations across different domains (see Fig. \ref{fig:intro}(a) and (b)), with proof provided in Section ~\ref{Sec:3.4}. To leverage this new discovery, in Section ~\ref{Sec:3.2}, we thoroughly decouple the traditional single-stage coupled forward diffusion process into a two-stage process, involving noise diffusion and residual diffusion. In Section ~\ref{Sec:3.3}, we introduce the decoupled reverse process involving the residual removal stage and denoising stage.

\begin{figure*}[t]  
    \centering
    \includegraphics[width=0.85\linewidth]{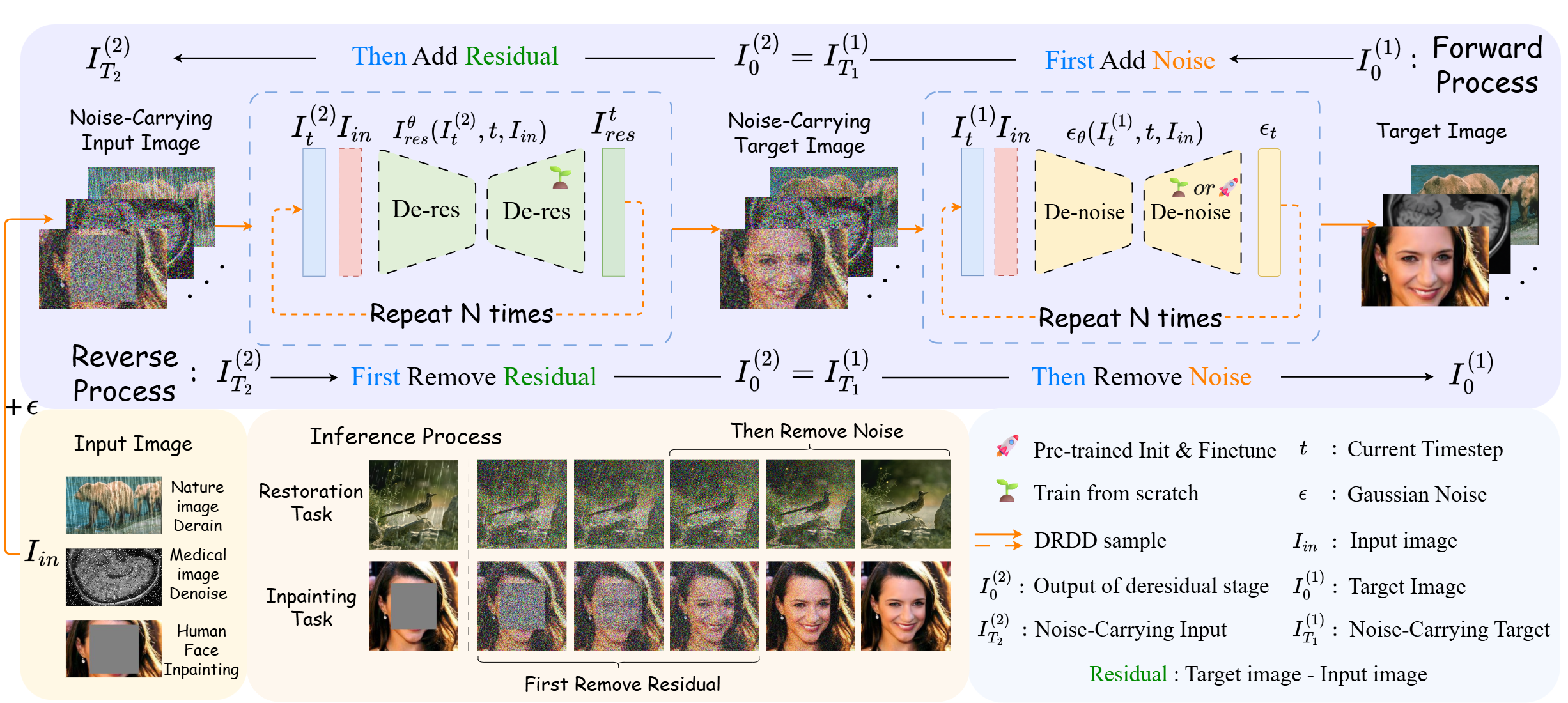}  
    \caption{\textit{Proposed DRDD framework.} DRDD decouples the traditional single forward diffusion processes into two sequential and independent process: a noise diffusion stage that injects Gaussian noise into the target image, followed by a residual diffusion stage that conducts deterministic target-to-source transformation, but now within a noise-carrying level.  The reverse diffusion process is correspondingly decoupled into a residual removal stage and a denoising stage.}
    \label{fig:main_fig}  
\end{figure*}
\subsection{The Role of Noise in Diffusion Models}\label{Sec:3.4}

Conventional functions of noise in diffusion models are moving data off low-dimensional manifolds and enriching training signals for score estimation \citep{song2019generative, lai2025principlesdiffusionmodels}, with the noise being controlled over time schedules. Beyond this, we discover that without time-step control, a certain level of fixed Gaussian noise can act as a ``domain harmonizer” in unified I2I tasks, minimizing the distribution gap of features across different domains. Here we give a mathematical expression:
 \begin{proposition}[]
    \label{proposition:noise_reduct_KL}
    Let \( P \) and \( Q \) be two distinct probability distributions over a space \( \mathcal{X} \). Suppose that we inject Gaussian noise \( \mathcal{N}(0, \sigma^2) \) (with \( \sigma \neq 0 \)) to both distributions and denote \( P_\sigma \) and \( Q_\sigma \) as the resulting distributions. Then, the Kullback-Leibler (KL) divergence between \( P_\sigma \) and \( Q_\sigma \) is less than the KL divergence between \( P \) and \( Q \):
    \begin{align}
    D_{\text{KL}}(P_\sigma \parallel Q_\sigma) < D_{\text{KL}}(P \parallel Q)
\label{eq:propo}
\end{align}
\end{proposition}

\noindent The proof is provided in Appendix \ref{Appendix:proposition}. Although this ``domain harmonizer” is particularly beneficial for unified I2I translation tasks, the conventional coupled diffusion process removes the injected noise simultaneously with residuals, thereby undermines such harmonizing benefit. 

\subsection{Decoupled Forward Process}\label{Sec:3.2}
DRDD decouples the forward diffusion process into two sequential and independent stages: a stochastic \textit{noise diffusion} followed by a deterministic \textit{residual diffusion}. As illustrated in the top row of Fig.~\ref{fig:main_fig}, the forward process starts from target image $I_0^{(1)}$, where we inject Gaussian noises to obtain noise-carrying target $I^{(1)}_{T_1}$. Then this is followed by a deterministic residual diffusion stage that models target-to-source mapping within a fixed-noise domain, leading to noise-carrying input image $I_{T_2}^{(2)}$.
 This entire process is visualized in Fig.~\ref{fig:main_fig} through ``Forward Process": $I_0^{(1)}\rightarrow I_{T_1}^{(1)}= I_{0}^{(2)}\rightarrow I_{T_2}^{(2)}$.

 Given a paired input image $I_{in}$ and target image $I_0^{(1)}$, the noise diffusion stage perturbs $I_0$ by progressively injecting Gaussian noise as: 
\vspace{-1mm}
\begin{align}
I^{(1)}_{t} &= I^{(1)}_{t-1} + \beta_{t} \varepsilon_{t-1} = I^{(1)}_0 + \bar\beta_{t}\,\varepsilon, \label{eq:2}
\end{align}
where $\varepsilon_{t-1},..., \varepsilon \sim \mathcal N(0,\mathbf I)$, 
and $\beta_{t}$ is the noise coefficient schedule that controls the noise diffusion speed ($\bar\beta_{t}\ = \sqrt{\sum\nolimits_{i=1}^{t}\beta_i^2}$). $I^{(1)}_{t}$ denotes the image at the forward diffusion step $t$ in the noise diffusion stage. This injected noise serves as "domain harmonizer" as well as utilizes the original manifold lifting ability. We then pass this terminal state of diffusion stage to the residual diffusion stage as the initial state by setting
\(I^{(2)}_0 := I^{(1)}_{T_1}\) (see Fig.~\ref{fig:main_fig}). 
We define residual as the difference between the source and target images: $I_{res} = I_{in} - I_{0}$.
The subsequent stage, residual diffusion, models the deterministic target-to-source process by injecting the residual $I_{res}$: 
\vspace{-3mm}
\begin{align}
I^{(2)}_{t} &= I^{(2)}_{t-1} + \alpha_{t} I_{\text{res}} 
= I^{(2)}_0 + \bar\alpha_{t} I_{\text{res}} 
\label{eq:3}
\end{align}
\noindent where $\alpha_{t}$ is the residual schedule coefficient ($\bar\alpha_{t}=\sum\nolimits_{i=1}^{t}\alpha_{i}\ $) and $I^{(2)}_{t}$ denotes the image at the forward diffusion step $t$ in the residual diffusion stage. When $t=T_2$ (total steps of residual diffusion) and $\bar\alpha_{T_2} = 1$, $I^{(2)}_{T_2}$ yields the final stage, completing the forward process:
\begin{small} 
\begin{align}
I^{(2)}_{T_2} = I^{(2)}_{0} + I_{\text{res}} 
= I^{(1)}_0 + \bar\beta_{T_1}\varepsilon + I_{\text{res}} = I_{in} + \bar\beta_{T_1}\varepsilon.
\label{eq:4}
\end{align}
\end{small}

\vspace{-1mm}
\subsection{Decoupled Reverse Process}\label{Sec:3.3}
Correspondingly, the reverse process is decoupled into two independent stages: residual-removal and denoising. Each stage is managed by a network trained with distinct objectives. As shown in Fig.~\ref{fig:main_fig}, the reverse process starts from noise-carrying image $I_{T_2}^{(2)}$, where DRDD first performs residual removal exclusively within the noise-carrying domain and obtains $I_{0}^{(2)}$. Then it conducts a
denoising process to transform noise-carrying target $I_{T_1}^{(1)}$ into clean target $I_{0}^{(1)}$. 
This entire reverse process is visualized in Fig.~\ref{fig:main_fig} through ``Reverse Process": $I_{T_2}^{(2)}\rightarrow I_{0}^{(2)}= I_{T_1}^{(1)}\rightarrow I_{0}^{(1)}$.

\begin{algorithm}[t]
\setstretch{0.48} \caption{Training Algorithm} \label{alg:algorithm1} \KwIn{Target image: $I_0$; Input image: $I_{in}$ \\ \quad\quad\quad Residual image: $I_{res} = I_{in} - I_{0}$.} 
\Repeat{\textnormal{converged}}{ $t \sim Uni({1,...T_1}), \epsilon \sim \mathcal{N}(\bm{0},\bm{I}) , I_{t}^{(1)} =I_0+\overline{\beta}_{t}\epsilon$; \\ Take gradient descent step on
\vspace{-1mm}
\[
\nabla_{\theta} = \| \epsilon - {\epsilon}_{\theta}(I_{t}^{(1)},t) \|_1
    \]
}
\Repeat{\textnormal{converged}}{ $t \sim Uni({1,...T_2}), I_t^{(2)} = I_{T_1}^{(1)}+\overline{\alpha}_{t}I_{res}$; \\ Take gradient descent step on  \\
\vspace{-1mm}
\[
    \nabla_{\theta} \| I_{\text{res}} - I_{\text{res}}^{\theta}(I_t^{(2)}, I_{\text{in}}, t) \|_1
    \]
}

\end{algorithm} 

In the residual-removal stage ($I_{T_2}^{(2)}\rightarrow I_{0}^{(2)}$), DRDD aims to remove residuals from $I^{(2)}_{T_2}$, which involves estimation of the residuals injected during the forward process, as described in Eq.~\ref{eq:3}.
To this end, we train a residual removal network denoted as $ I^{\theta}_{res}(I_{t}^{(2)},t,I_{in})$. Given the current image $I_{t}^{(2)}$, timestep $t$ and the degraded image $I_{in}$, the network learns to predict residuals in a noise-carrying domain. 
 Using Eq.~\ref{eq:3}, we obtain the estimated target images $I_0^{(2)}(\theta)=I_t^{(2)}-\bar{\alpha}_t I_{res}^\theta$. Given $I_0^{(2)}(\theta)$ and $I_{res}^{\theta}$, the generation process of residual removal is defined as:
\vspace{-2mm}
\begin{align}
p_\theta(I_{t-1}^{(2)}\mid I_{t}^{(2)})
&:= q(I_{t-1}^{(2)}\mid I_{t}^{(2)}, I_0^{(2)}(\theta), I_{\text{res}}^\theta) \notag \\
&= \mathcal{N}\left(I_{t-1}^{(2)}; I_0^{(2)}(\theta) + \bar{\alpha}_{t-1} I_{\text{res}}^\theta,\, 0\right).
\label{eq:5}
\end{align}
\noindent With Eq.~\ref{eq:3} and \ref{eq:5}, $I_{t-1}$ can be sampled from $I_{t}$ via:
\vspace{-2mm}
\begin{align}
    I_{t-1}^{(2)}  = I_{t}^{(2)} - \alpha_{t} I_{res}^\theta(I_{t}^{(2)}, I_{in}, t).
    \label{eq:6}
\end{align}
By performing the core source-to-target mapping before any noise is removed, the domain-harmonizing and manifold lifting effects are preserved, thereby substantially simplifying the learning of a unified image mapping.

In the denoising stage ($I_{T_1}^{(1)}\rightarrow I_{0}^{(1)}$), we train a denoise network ${\epsilon}_{\theta}$ which learns to remove Gaussian noises. Using Eq.~\ref{eq:2}, we obtained the estimated target image $I_0^{(1)}(\theta) = I_{t}^{(1)} - \bar\beta_{t}\,\epsilon_{\theta}$. Given 
$I^{(1)}_{t}$ and model prediction of the noise ${\epsilon}_{\theta}$, we factor this variational distribution $q_\sigma\ $ as:
\vspace{-2mm}
\begin{align}
\scriptsize 
p_\theta\big(I^{(1)}_{t-1} \mid I^{(1)}_{t}\big)
&:= q_\sigma\big(I^{(1)}_{t-1} \mid I^{(1)}_t, I_0^{(1)}(\theta)\big) \nonumber \\
&\!\!\!\!\!\!\!\!\!\!\!\!\!\!\!\!\!\!\!\!\!\!\!\!= \mathcal{N} (I_{t-1};\sqrt{\bar{\beta }_{t-1}^2 -\sigma _t^2} \frac{(I^{(1)}_t - I_0^{(1)}(\theta))}{\bar{\beta }_{t}} ,\sigma _{t}^2\mathbf{I}
 ),
\label{eq:7}
\end{align}
\noindent where $\sigma _t^2=\eta \beta_t^2\bar{\beta} _{t-1}^2/\bar{\beta} _{t}^2$ and $\eta$ controls whether the generation process is random ($\eta=1$) or deterministic ($\eta=0$). With Eq.~\ref{eq:2} and \ref{eq:7}, the iterative process is (see Appendix~\ref{Appendix:DRDD}):
\vspace{-2mm}
\begin{small} 
\begin{align}
I_{t-1}^{(1)}= I_{t}^{(1)} -(\bar{\beta }_t-\sqrt{\bar{\beta }_{t-1}^2-\sigma _t^2} )\, 
{\epsilon}_{\theta}(I_{t}^{(1)}, t) +\sigma_t\varepsilon_t.\
\end{align}
\end{small} 
\noindent The complete sampling algorithm is shown in Alg. \ref{alg:algorithm2}. 

\begin{table*}[t]
    \centering
    \caption{\textit{Performance comparisons of five unified multi-task image restoration tasks on All-in-One-5 dataset \citep{instructIR}.} Denoising results are reported at the noise level $\sigma$ = 25. SSIM ($\uparrow$), LPIPS ($\downarrow$) and FID ($\downarrow$) are reported. Best results are highlighted in \textcolor{red}{\textbf{red}}, while the second-best results are \textcolor{blue}{blue}.  Diffusion-based methods are denoted by ``*''. Our DRDD demonstrates superior or competitive performance compared to recent models, especially in perceptual metrics. Due to space limitation, PSNR results and computational costs are provided in Appendix~\ref{app:cost}}.
    \vspace{-2mm}
    \resizebox{\textwidth}{!}{
    \begin{tabular}{l|c|c|c|c|c|c}
        \toprule[0.15em]
        \multirow{2}{*}{\textbf{Method}} &
        \textbf{Low-Light} &
        \textbf{Deraining} &
        \textbf{Denoising} &
        \textbf{Deblurring} &
        \textbf{Dehazing} &
        \textbf{Average} \\
        \cmidrule(lr){2-2}\cmidrule(lr){3-3}\cmidrule(lr){4-4}\cmidrule(lr){5-5}\cmidrule(l){6-6}\cmidrule(lr){7-7}
        & 
        SSIM / LPIPS / FID &
        SSIM / LPIPS / FID &
        SSIM / LPIPS / FID &
        SSIM / LPIPS / FID &
        SSIM / LPIPS / FID &
        SSIM / LPIPS / FID \\
        \midrule[0.1em]

        DA-CLIP* \citep{luo2024daclip}&
        0.819 / \textcolor{blue}{0.115} / \textcolor{blue}{36.2} &
        0.973 / 0.169 / 9.25 &
        0.809 / 0.108 / \textcolor{blue}{34.4} &
        0.829 / \textcolor{blue}{0.135} / \textcolor{blue}{16.2} &
        0.959 / 0.015 / \textcolor{blue}{3.82} &
        0.876 / 0.108 / \textcolor{blue}{20.0} \\
        \midrule[0.05em]
        DiffuIR* \citep{zheng2024diffuir}&
        0.804 / 0.204 / 77.7 &
        0.961 / 0.042 / 18.3 &
        0.856 / 0.114 / 34.9 &
        0.793 / 0.182 / 24.1 &
        0.930 / 0.046 / 13.6 &
        0.869 / 0.117 / 33.7 \\
        \midrule[0.05em]
        AdAIR \citep{cui2025adair} &
        \textcolor{blue}{0.844} /  0.120 /  48.9 &
        \textcolor{blue}{0.978} / 0.015 / 8.73 &
        0.888 / 0.109 / 39.2 &
        0.857 / 0.189 / 19.9 &
        0.975 / 0.015 / 14.5 &
        0.909 / 0.089 / 26.1 \\
        \midrule[0.05em]
        VLUNet \citep{vlunet} &
        0.832 / 0.144 / 60.4 &
        \textcolor{red}{\textbf{0.981}} / \textcolor{red}{\textbf{0.012}} / \textcolor{red}{\textbf{6.17}} &
        \textcolor{blue}{0.890} / 0.098 / 36.4 &
        0.840 / 0.214 / 24.0 &
        \textcolor{red}{\textbf{0.979}} / \textcolor{red}{\textbf{0.012}} / 12.9 &
        0.904 / 0.096 / 27.9 \\
        \midrule[0.05em]
        DFPIR \citep{tian2025dfpir} &
        0.843 / 0.122 / 50.6 &
        0.977 / 0.017 / 8.32 &
        0.889 / \textcolor{red}{0.091} / 35.0 &
        \textcolor{blue}{0.873} / 0.164 / 17.0 &
        \textcolor{blue}{0.978} / \textcolor{blue}{0.013} / 13.6 &
        \textcolor{blue}{0.912} / \textcolor{blue}{0.081} / 24.9 \\
        \midrule[0.1em]
        \textbf{DRDD(Ours)*} &
      \textcolor{red}{\textbf{0.864}} / \textcolor{red}{\textbf{0.103}} / \textcolor{red}{\textbf{35.4}} &
        \textcolor{blue}{0.978} / \textcolor{blue}{0.014} / \textcolor{blue}{8.06} &
        \textcolor{red}{\textbf{0.893}} / \textcolor{blue}{0.097} / \textcolor{red}{\textbf{28.9}} &
        \textcolor{red}{\textbf{0.881}} / \textcolor{red}{\textbf{0.134}} / \textcolor{red}{\textbf{15.8}} &
        0.972 / \textcolor{red}{\textbf{0.012}} / \textcolor{red}{\textbf{3.54}} &
        \textcolor{red}{\textbf{0.916}} / \textcolor{red}{\textbf{0.073}} / \textcolor{red}{\textbf{18.3}} \\
        \bottomrule[0.1em]
    \end{tabular}
    }
    \vspace{-2mm}
    \label{tab:AllinOne5}
\end{table*}

\begin{algorithm}[t]
\setstretch{0.51}
\caption{Sampling Algorithm}
\label{alg:algorithm2}
\KwIn{Input image: $I_{in}$. \\}
$\epsilon \sim \mathcal{N}(\bm{0}, \bm{I})$; \\
$I_{T_2}^{(2)} = I_{in} + \overline{\beta}_{T_1} \epsilon$; \\
\For {$t=T_2,\dots,1 $}{
    \quad $I_{t-1}^{(2)} = I_{t}^{(2)} - \alpha_t I_{res}^{\theta}(I_{t}^{(2)}, I_{in}, t)$; \\
}
$I_{T_1}^{(1)} = I_{0}^{(2)}$; \\
\For{$t = T_1, \dots, 1$} {
    $I_{t-1}^{(1)} = I_t^{(1)} - (\bar{\beta}_t - \sqrt{\bar{\beta}_{t-1}^2 - \sigma_t^2}) \cdot \epsilon_{\theta}(I_t^{(1)}, t) + \sigma_t \varepsilon_t$;
}

\Return {$I_0^{(1)}$}
\end{algorithm}

\paragraph{Training Objectives.} 
We derive the following simplified loss function for training the residual removal network $I_{res}^{\theta}$ and the denoising network $\epsilon_\theta$ (see proofs in Appendix~\ref{Appendix:loss}):
\begin{align}
\mathcal{L}_{\text{res}}(\theta) 
&= \mathbb{E} \left[ \left\| I_{\text{res}} -  I_{\text{res}}^\theta(I_t^{(2)},t,I_{in}) \right\|_1 \right].\\
\mathcal{L}_{\epsilon}(\theta) &= \mathbb{E} \left[ \left\| \epsilon - {\epsilon}_{\theta}(I_t^{(1)}, t) \right\|_1 \right].
\label{eq:10}
\end{align}


\noindent The complete training algorithm is shown in Alg.~\ref{alg:algorithm1}. According to Eq.~\ref{eq:10}, since denoising network is trained solely on clean images with no need of corresponding source domain images, DRDD significantly enhances data efficiency. Besides, we can initialize the denoising network with weights pretrained on large scale natural image datasets. Although our derivation builds upon the DDPM \citep{NEURIPS2020_DDPM} and DDIM \citep{song2022ddim} framework, it is also compatible with score-based SDE \citep{song2021scorebased} methods. The detailed derivation is in Appendix~\ref{App:sde}.



\section{Experiments}
\label{others}
In the experiments, we first validate DRDD's effectiveness in unified I2I translation through two challenging settings: (1) multi-task benchmarks (All-in-One-5 \citep{instructIR} and CDD-11 \citep{onerestore}), where a single model handles multiple distinct restoration tasks; and (2) cross-domain single I2I translation task using our self-collected MNMD benchmark, where images from disparate domains within a single task. Next, we show that DRDD's advantages extend to standard single I2I translation tasks, where its domain harmonization mitigates instance-level distribution shifts. We then conduct data efficiency analyses, verifying DRDD's robust performance with limited paired data on both task-specific and unified settings. Further, we demonstrate DDRD's compatibility with other diffusion backbones. Finally, we theoretically and empirically investigate the optimal noise injection level for DRDD.

\subsection{Experiment Settings}
\textbf{Datasets.} We use a diverse set of widely-used datasets across various I2I translation tasks. For unified image restoration, we experiment on All-in-One-5~\citep{instructIR} and CDD-11~\citep{onerestore}, which includes various restoration tasks. We also construct an MNMD benchmark, which contains various types of noise and covers multiple domains. We use All-in-One-3~\citep{instructIR} and Low-Light~\citep{lol} to validate our data efficiency. Studies on inpainting (CelebA-HQ \citep{celebahq}), super-resolution (FFHQ \citep{ffhq}) and other task-specific I2I translation further demonstrate our framework's capability on single task I2I translation. The configuration of the aforementioned datasets is detailed in Appendix~\ref{app:more_datasets}.


\noindent\textbf{Model Architecture.} 
Following design in ADM~\citep{dhariwal2021diffusion}, the denoising network in DRDD, DiffUIR~\citep{zheng2024diffuir} and RDDM~\citep{liu2024rddm} adopt the same UNet backbone and hyperparameter settings, where the channel depth is $C = 128$ with channel multiplier = $(1, 1, 2, 2, 4, 4)$. Our denoising model follows the U-Net architecture in~\citep{dhariwal2021diffusion}, where the channel depth is $C = 64$ with channel multiplier = $(1, 2, 4, 8)$.

As for inference, we adopt the DDIM\citep{song2022ddim} sampling strategy as described by \citep{song2022ddim}, with the sampling step size set to 2 for both the denoising and residual removal stages in all experiments. We provide inference steps and corresponding model performance in Appendix~\ref{app:sampling_steps}.
\begin{figure}[ht]
\centering
\includegraphics[width=0.90\linewidth]{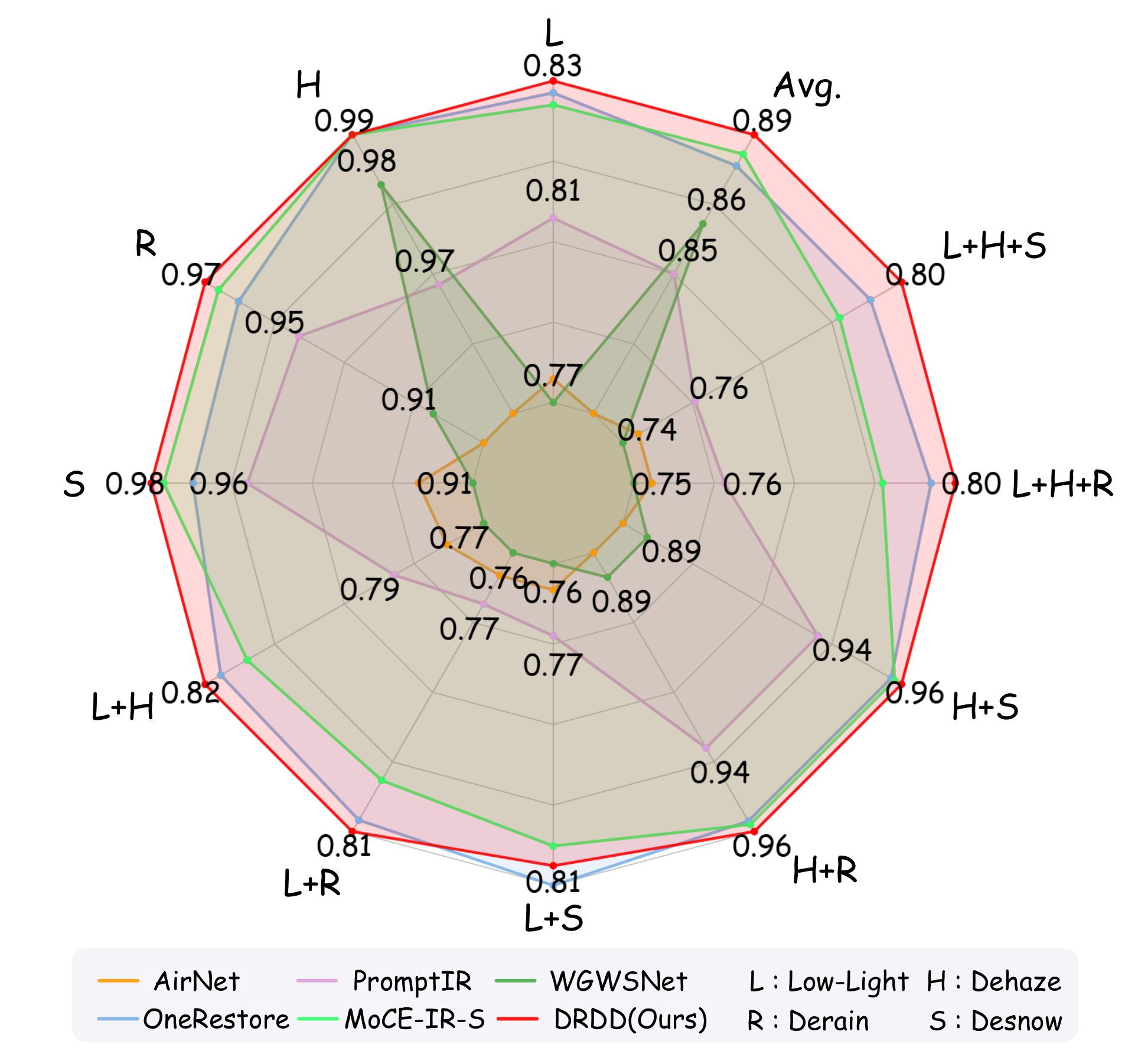}
\caption{\textit{Quantitative comparison to state-of-the-art on 11 degradation tasks and their average.} SSIM ($\uparrow$) is reported. Our DRDD method consistently outperforms recent SOTA models, with favorable results in complex composited degradation scenarios. All experiments are conducted on the CDD11 dataset \citep{onerestore}.}
\label{fig:CDD11}
\vspace{-1em}
\end{figure}

\vspace{-0.5em}
\begin{figure*}[h]
    \centering
    \includegraphics[width=0.99\linewidth]{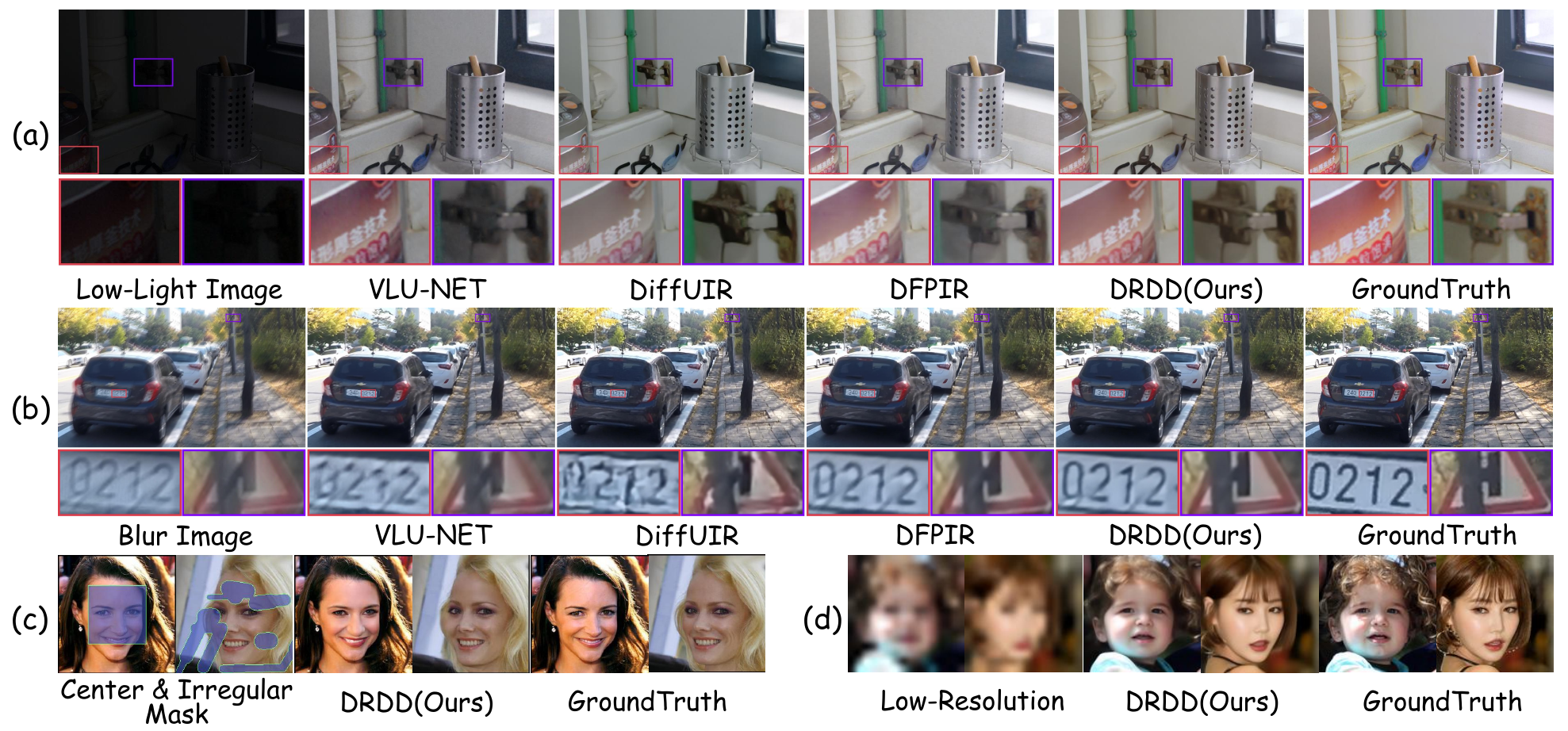}  
    \vspace{-2mm}
    \caption{\textit{Visual results of state-of-the-art methods and our proposed DRDD.} (a) Comparison of low-light enhancement results on the LoLV1 dataset \citep{lol}. (b) Comparison of blur restoration results on the GoPro dataset~\citep{gopro}. (c) Face inpainting results (center and irregular mask) in CelebA-HQ~\citep{celebahq}. (d) Super-Resolution result in FFHQ~\citep{ffhq}. Zoom in for best view. More visual results are provided in Appendix~\ref{Appendix:d}.} 
    \label{fig:main_result}   
\end{figure*}

\subsection{Performance on I2I Translation Tasks}
To comprehensively evaluate the unified capability of DRDD on I2I translation tasks, we design experiments from three perspectives: handling multiple restoration tasks, performing a single task across multiple domains, and conducting task-specific I2I within a single domain.
\vspace{-4mm}
\paragraph{Multi-task Unified Restoration.} Following recent studies \citep{tian2025dfpir, vlunet}, we evaluate the effectiveness of our method in handling various image degradation types on the \textbf{All-in-One-5} benchmark (see Appendix.~\ref{app:AiO5-dataset}). As shown in Tab.~\ref{tab:AllinOne5}, DRDD achieves state-of-the-art (SOTA) performance on most image restoration tasks. Notably, DRDD consistently outperforms both recent diffusion-based approaches (DA-CLIP \citep{luo2024daclip}, DiffuIR \citep{dcfdiff2025}) and other non-diffusion based (DFPIR \citep{tian2025dfpir}, VLUNet \citep{vlunet}, and AdAIR \citep{cui2025adair}) in all three metrics. This advantage is also clearly reflected in Fig.~\ref{fig:main_result}-a and Fig.~\ref{fig:main_result}-b, where DRDD produces restorations with richer details and fewer artifacts compared to other methods. 

The \textbf{CDD-11} dataset \citep{onerestore}, which is used to assess robustness and generalization under complex scenarios, contains 11 different degradation types (Appendix.~\ref{app:CDD11}). Fig.~\ref{fig:CDD11} presents a comparison between DRDD and five recent approaches \citep{potlapalli2023promptir, airnet, wgwsnet, onerestore, MOCEIRS} in CDD-11. DRDD consistently outperforms recent SOTA models across most degradation categories, achieving the highest average SSIM score. In particular, DRDD demonstrates clear advantages in challenging composite scenarios (e.g., the L+H+S and L+H+R in Fig.~\ref{fig:CDD11}), where other methods such as PromptIR \citep{potlapalli2023promptir}, WGWSNet \citep{wgwsnet}, and MoCE-IR-S \citep{MOCEIRS} experience notable performance drops. These results demonstrate the superiority of the DRDD framework on unified image restoration benchmarks.

\begin{table}[t]
    \centering
    \caption{\textit{Performance comparison of several methods on MNMD dataset}. Best results are highlighted in \textbf{Bold}.}
    \vspace{-2mm}
    \resizebox{\linewidth}{!}{
    \begin{tabular}{l*{8}{c}}
        \toprule[0.15em]
        \multirow{2}{*}{\textbf{Method}} &
        \multicolumn{2}{c}{\textbf{Natural}} &
        \multicolumn{2}{c}{\textbf{Medical}} &
        \multicolumn{2}{c}{\textbf{Remote}} &
        \multicolumn{2}{c}{\textbf{Average}} \\
        \cmidrule(lr){2-3}\cmidrule(lr){4-5}\cmidrule(lr){6-7}\cmidrule(lr){8-9}
        & SSIM$\uparrow$ & LPIPS$\downarrow$
        & SSIM$\uparrow$ & LPIPS$\downarrow$
        & SSIM$\uparrow$ & LPIPS$\downarrow$
        & SSIM$\uparrow$ & LPIPS$\downarrow$ \\
        \midrule[0.1em]
        RDDM \citep{liu2024rddm}   & 0.8333 & 0.1703 & 0.8343 & 0.1917 & 0.8542 & 0.1485 & 0.8406 & 0.1702 \\
        \midrule[0.1em]
        IR-SDE \citep{luo2023IRSDE}  & 0.8062 & 0.1129 & 0.8492 & 0.0510 & 0.8090 & 0.0999 & 0.8215 & 0.0879 \\
        \midrule[0.1em]
        VLUNET \citep{vlunet} & 0.9308 & 0.0782 & 0.9267 & 0.0840 & 0.9249 & 0.0710 & 0.9274 & 0.0784 \\
        \midrule[0.1em]
        \textbf{DRDD(Ours)} & \textbf{0.9391} & \textbf{0.0492} & \textbf{0.9324} & \textbf{0.0629} & \textbf{0.9300} & \textbf{0.0539} & \textbf{0.9338} & \textbf{0.0553}  \\
        \bottomrule[0.1em]
    \end{tabular}
    }
    \vspace{-2mm}
    \label{tab:MDMT}
\end{table}

\vspace{-4mm}
\paragraph{Single Task I2I Translation in Multi-Domain.} Recent All-in-One models often neglect performance drops from domain shifts. We believe that residual removal in the noise-carrying domain reduces degradation conflicts, while serving as a domain harmonizer that aligns feature representations across domains and mitigates domain gap. Therefore, we focus on a multi-domain image denoising task and establish a challenging multi-domain benchmark by adding different kinds of noise to natural (WED+BSD400 \citep{rain100}), remote sensing (UC-Merced \citep{ucmerced}), and medical (BrainWeb \citep{brainweb1, brainweb2, brainweb3}) image datasets. As shown in Tab.~\ref{tab:MDMT}, we compare our proposed DRDD model with several SOTA restoration methods, including RDDM \citep{liu2024rddm}, IR-SDE \citep{luo2023IRSDE}, and VLUNET \citep{vlunet}. Across all domains, our model consistently achieves the highest SSIM and lowest LPIPS on natural, medical, and remote sensing images. This demonstrates that DRDD effectively handles multi-domain image restoration within a single task. The complete benchmark construction pipeline is detailed in the Appendix~\ref{app:MNMD-dataset}.



\vspace{-2mm}
\paragraph{Single-Task I2I Translation in Single Domain.} 
Distribution gaps can occur even within a single I2I translation task in one domain due to diverse input characteristics.
Here, we further verify the effectiveness of DRDD on single-task I2I translation problems, where the data is sourced from a single domain. Specifically, we conduct experiments on image inpainting, super-resolution, and other single tasks, such as image deraining and low-light enhancement. For image inpainting,  DRDD is compared with CTSDG~\citep{guo2021image}, MISF~\citep{li2022misf}, and TransRef~\citep{liu2025transref} on the CelebA-HQ~\citep{celebahq} dataset under various mask patterns and resolutions. See Fig.~\ref{fig:main_result}-c and Appendix~\ref{app:add_ex} for all the quantitative and qualitative comparisons. These extensive results demonstrate that DRDD also achieves superior performance on single I2I translation tasks. 


\vspace{-1mm}
\subsection{Performance on Limited Training Data}
We now validate another key advantage of our decoupled design: its ability to achieve superior performance with limited paired data. 
We randomly sub-sample the training set to {75\%, 50\% and 25\%} while keeping the validation set fixed. As shown in Fig.~\ref{fig:data_efficiency}, on two representative datasets (Low-Light and All-in-One-3, see details in Appendix~\ref{app:All-in-One-3}), DRDD achieves better performance than existing baselines, especially under limited data conditions. Here we initialize the training of denoising network with pretrained weights on ImageNet~\citep{imagenet}. Notably, as the amount of training data decreases, the relative performance drop of DRDD remains substantially smaller compared to other methods, underscoring the data efficiency of our approach. These results demonstrate that DRDD can effectively maintain restoration quality even with severely reduced training data, highlighting its practical superiority in data-constrained cases.


\begin{figure}
    \centering
    \begin{subfigure}[t]{0.95\linewidth}
        \centering
        \includegraphics[width=\linewidth]{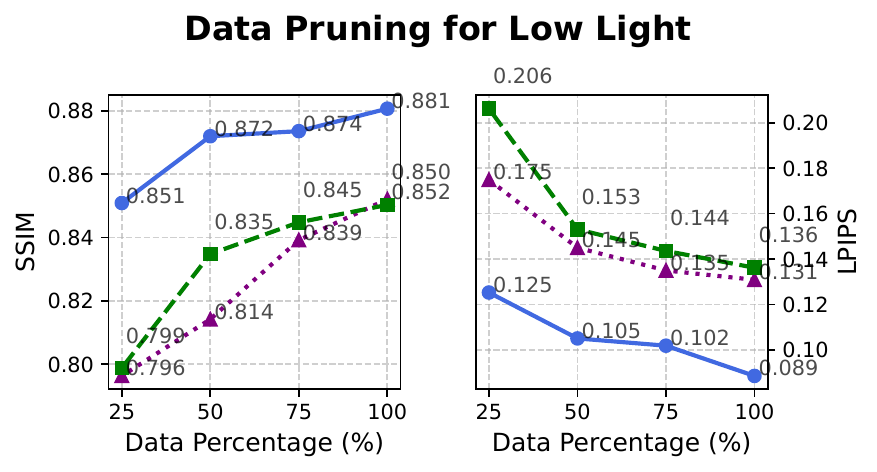}
        \label{fig:main_result_a}
    \end{subfigure}
    \vspace{-1em}
    \begin{subfigure}[t]{0.95\linewidth}
        \centering
        \includegraphics[width=\linewidth]{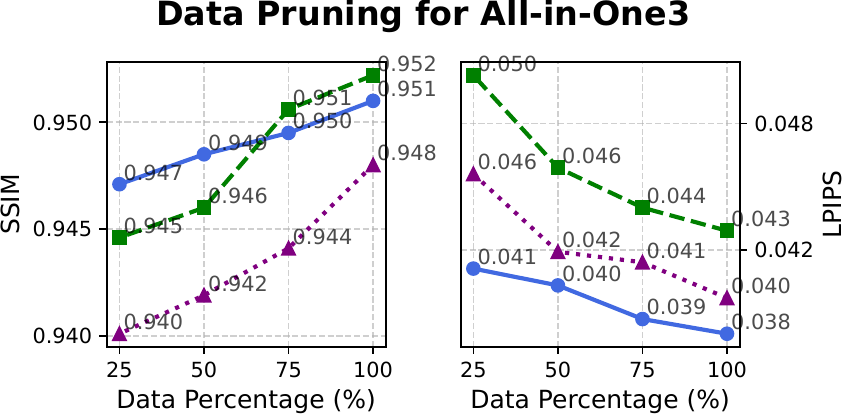}
        \label{fig:main_result_b}
        \vspace{-1em}
    \end{subfigure}
    \begin{subfigure}[t]{0.80\linewidth}
        \centering
        \includegraphics[width=\linewidth]{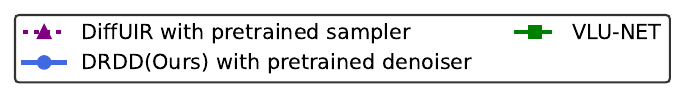}
        \label{fig:main_result_c}
    \end{subfigure}
    \vspace{-2em}
    \caption{\textit{Data Pruning on All-in-One-3 \citep{instructIR} and Low-Light dataset.} SSIM ($\uparrow$) and LPIPS ($\downarrow$) are reported. As the training data decreases, DRDD's performance drop is much smaller than other methods.}
    \vspace{-1em}
    \label{fig:data_efficiency}
\end{figure}

\subsection{Extensions to other Diffusion Paradigms}
We further incorporate the decoupling strategy into a SDE-based diffusion model built upon IR-SDE~\citep{luo2023IRSDE} to ensure the compatibility of our framework.
As shown in Tab.~\ref{tab:5}, the decoupled SDE consistently outperforms the baseline SDE on two important tasks: deraining and inpainting. In addition, it achieves comparable performance in denoising while obtaining better FID scores. 
These results indicate that the decoupling mechanism can be effectively extended to other diffusion frameworks.

\subsection{Investigation of Noise Injection Level}
The reverse process of DRDD starts from a noise-carrying image, where a Gaussian noise with a predefined and fixed noise level is injected. The intensity of such Gaussian noise is critical to overall performance (see Proposition \ref{proposition:noise_reduct_KL}). Here, we explore this issue from both theoretical and experimental perspectives. We define two distinct distances:
\begin{equation}
A(\sigma) = \Delta(P^\sigma_s, P^\sigma_t), \quad B(\sigma) = \Delta(P^\sigma_s, P_s)
\end{equation}
where \(A(\sigma)\) quantifies the distance between noise-carrying target distribution \(P_\sigma^t\) and the noise-carrying source distribution \(P_\sigma^s\) while \(B(\sigma)\) quantifies the distance between the noise-carrying source distribution \(P_\sigma^s\) and the original source distribution \(P_s\).
\(\Delta\) denotes the Maximum Mean Discrepancy (MMD) in this case, and the full formula is provided in the Appendix~\ref{Appendix:Noise-injection}. As the noise level \( \sigma \) increases, both \( A(\sigma) \) and \( B(\sigma) \) increase monotonically. We aim to find a noise level that \( A(\sigma) \) is small enough (since a larger distance complicates translation), and \( B(\sigma) \) is also minimized (to prevent significant input corruption), which leads to:
\begin{equation}
J(\sigma; \lambda) = \lambda \, \widetilde{A}(\sigma) + (1 - \lambda) \, \widetilde{B}(\sigma), \quad \lambda \in [0, 1].
\label{eq:12}
\end{equation}
where \(\widetilde{A}(\sigma)\) and \(\widetilde{B}(\sigma)\) represent the normalized values of \(A(\sigma)\) and \(B(\sigma)\), respectively. By minimizing this objective function, we obtain the optimal noise level \(\sigma^{\star}_J = \arg\min_J(\sigma; \lambda)\). The value of \(\lambda\) can be adjusted based on the desired balance between the two distances. We calculate the results on the All-in-One-5 datasets, taking \(\lambda=0.5\), and the obtained optimal $\bar{\beta}$ (noise intensity) is around 1.1 to 1.2. 
To further validate our Eq.~\ref{eq:12}, we conduct quantitative experiments across different noise levels. As shown in Fig.~\ref{fig:noise_scale}, the models achieve optimal performance when the noise intensity is set to 1.0, with stable and superior results observed in the range of 0.8 to 1.3. These findings align with our theoretical expectations.

\begin{table}[t]
    \centering
    \caption{\textit{Performance comparison of decoupled and coupled SDE-based diffusion methods on single task I2I.} Results are evaluated on the CelebA-HQ \citep{celebahq}, Rain100 \citep{rain100}, and BSD400 \citep{bsd400} datasets.}
    \vspace{-2mm}
    \resizebox{\linewidth}{!}{
    \begin{tabular}{lccc ccc cc}
        \toprule[0.15em]
        \multirow{2}{*}{\textbf{Method}} &
        \multicolumn{2}{c}{\textbf{Inpainting}} &
        \multicolumn{3}{c}{\textbf{Deraining}} &
        \multicolumn{3}{c}{\textbf{Denoise} }\\
        \cmidrule(lr){2-3}\cmidrule(lr){4-6}\cmidrule(lr){7-9}
        &
        LPIPS$\downarrow$ &
        FID$\downarrow$ &
        PSNR$\uparrow$ &
        SSIM$\uparrow$ &
        LPIPS$\downarrow$ &
        SSIM$\uparrow$ &
        LPIPS$\downarrow$ &
        FID$\downarrow$ \\
        \midrule[0.1em]

        IR-SDE~\citep{luo2023IRSDE} &
        {0.0517} & {15.14} & {27.2} & {0.856} & {0.083} & \textbf{0.833} & \textbf{0.1014} & {33.29}\\
        \midrule[0.1em]
        \textbf{De-IRSDE(Ours)} &
        \textbf{0.0490} & \textbf{15.10} & \textbf{28.1} & \textbf{0.862} & \textbf{0.076} & {0.827} & {0.1069} & \textbf{31.87} \\
        \bottomrule[0.1em]
    \end{tabular}
    }
    \vspace{-1.5em}
    \label{tab:5}
\end{table}

\begin{figure}[ht]
\centering
\includegraphics[width=0.99\linewidth]{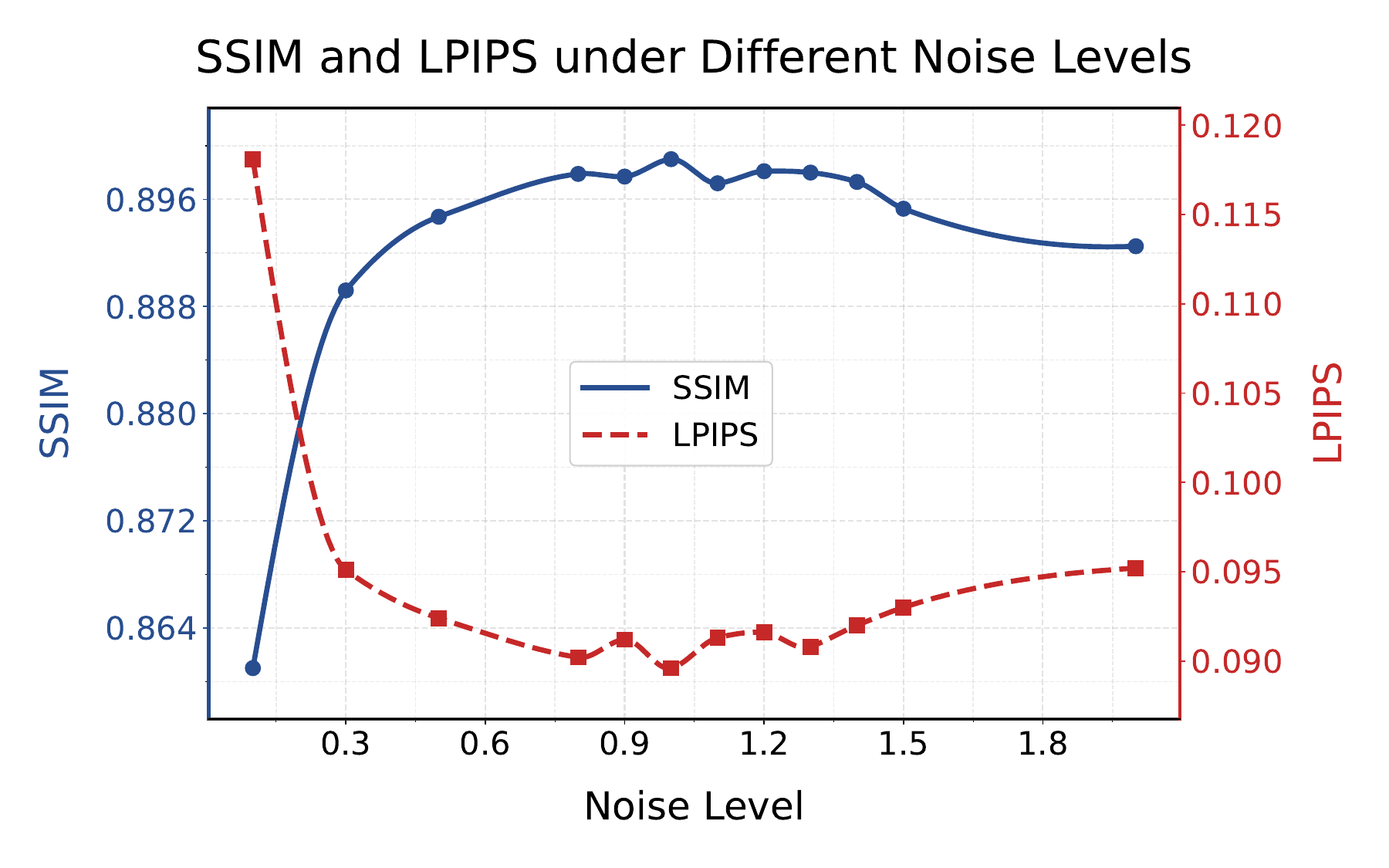}
\vspace{-5mm}
\caption{\textit{Performance comparison on the All-in-One-5 dataset\citep{instructIR} under varying noise injection level.}}
\label{fig:noise_scale}
\end{figure}

\vspace{-5mm}
\section{Conclusion}
In this paper, we present DRDD, a novel diffusion model that decouples standard diffusion into sequential noise diffusion and residual diffusion stages. Our work discovers a novel role of Gaussian noise as a ``domain harmonizer", which leads us to rethink conventional coupled diffusion models and propose a novel decoupled framework. By leveraging its decoupled mechanism, DRDD not only simplifies the learning of a unified mapping across tasks but also enables the denoising stage to be trained exclusively on unpaired images. Comprehensive theoretical and empirical analyses demonstrate DRDD's effectiveness. We believe this work opens new perspectives on noise utilization in generative models and provides a solid foundation for unified image translation systems.



\section{Acknowledgements}
This work was supported by National Natural Science Foundation of China (62306253, T2522030), Early Career Scheme from the Research Grants Council of Hong Kong SAR (27207025, 27204623), Guangdong Natural Science Fund-General Programme (2024A1515010233), China Postdoctoral Science Foundation under Grant Number 2025M781669, and the Fundamental Research Project of SIA (2025JC1K05).

{
    \small
    \bibliographystyle{ieeenat_fullname}
    \bibliography{main}

\begin{thebibliography}{73}
\providecommand{\natexlab}[1]{#1}
\providecommand{\url}[1]{\texttt{#1}}
\expandafter\ifx\csname urlstyle\endcsname\relax
  \providecommand{\doi}[1]{doi: #1}\else
  \providecommand{\doi}{doi: \begingroup \urlstyle{rm}\Url}\fi

\bibitem[Abdelhamed et~al.(2018)Abdelhamed, Lin, and Brown]{abdelhamed2018high}
Abdelrahman Abdelhamed, Stephen Lin, and Michael~S Brown.
\newblock A high-quality denoising dataset for smartphone cameras.
\newblock In \emph{Proceedings of the IEEE conference on computer vision and pattern recognition}, pages 1692--1700, 2018.

\bibitem[Arbelaez et~al.(2010)Arbelaez, Maire, Fowlkes, and Malik]{bsd400}
Pablo Arbelaez, Michael Maire, Charless Fowlkes, and Jitendra Malik.
\newblock Contour detection and hierarchical image segmentation.
\newblock \emph{IEEE transactions on pattern analysis and machine intelligence}, 33\penalty0 (5):\penalty0 898--916, 2010.

\bibitem[Bansal et~al.(2022)Bansal, Borgnia, Chu, Li, Kazemi, Huang, Goldblum, Geiping, and Goldstein]{bansal2022cold}
Arpit Bansal, Eitan Borgnia, Hong-Min Chu, Jie~S Li, Hamid Kazemi, Furong Huang, Micah Goldblum, Jonas Geiping, and Tom Goldstein.
\newblock Cold diffusion: Inverting arbitrary image transforms without noise.
\newblock \emph{arXiv preprint arXiv:2208.09392}, 2022.

\bibitem[Bishop and Nasrabadi(2006)]{bishop2006pattern}
Christopher~M Bishop and Nasser~M Nasrabadi.
\newblock \emph{Pattern recognition and machine learning}.
\newblock Springer, 2006.

\bibitem[Chen et~al.(2019)Chen, Xu, Yang, Song, and Tao]{gatedgan}
Xinyuan Chen, Chang Xu, Xiaokang Yang, Li Song, and Dacheng Tao.
\newblock Gated-gan: Adversarial gated networks for multi-collection style transfer.
\newblock \emph{IEEE Transactions on Image Processing}, 28\penalty0 (2):\penalty0 546–560, 2019.

\bibitem[Choi et~al.(2021)Choi, Kim, Jeong, Gwon, and Yoon]{choi2021ilvrconditioningmethoddenoising}
Jooyoung Choi, Sungwon Kim, Yonghyun Jeong, Youngjune Gwon, and Sungroh Yoon.
\newblock Ilvr: Conditioning method for denoising diffusion probabilistic models, 2021.

\bibitem[Collins et~al.(2002)Collins, Zijdenbos, Kollokian, Sled, Kabani, Holmes, and Evans]{brainweb2}
D~Louis Collins, Alex~P Zijdenbos, Vasken Kollokian, John~G Sled, Noor~Jehan Kabani, Colin~J Holmes, and Alan~C Evans.
\newblock Design and construction of a realistic digital brain phantom.
\newblock \emph{IEEE transactions on medical imaging}, 17\penalty0 (3):\penalty0 463--468, 2002.

\bibitem[Conde et~al.(2024)Conde, Geigle, and Timofte]{instructIR}
Marcos~V Conde, Gregor Geigle, and Radu Timofte.
\newblock Instructir: High-quality image restoration following human instructions.
\newblock In \emph{Proceedings of the European Conference on Computer Vision (ECCV)}, 2024.

\bibitem[Cui et~al.(2025)Cui, Zamir, Khan, Knoll, Shah, and Khan]{cui2025adair}
Yuning Cui, Syed~Waqas Zamir, Salman Khan, Alois Knoll, Mubarak Shah, and Fahad~Shahbaz Khan.
\newblock Ada{IR}: Adaptive all-in-one image restoration via frequency mining and modulation.
\newblock In \emph{The Thirteenth International Conference on Learning Representations}, 2025.

\bibitem[Delbracio and Milanfar(2024)]{indi}
Mauricio Delbracio and Peyman Milanfar.
\newblock Inversion by direct iteration: An alternative to denoising diffusion for image restoration, 2024.

\bibitem[Deng et~al.(2009)Deng, Dong, Socher, Li, Li, and Fei-Fei]{imagenet}
Jia Deng, Wei Dong, Richard Socher, Li-Jia Li, Kai Li, and Li Fei-Fei.
\newblock Imagenet: A large-scale hierarchical image database.
\newblock In \emph{2009 IEEE Conference on Computer Vision and Pattern Recognition}, pages 248--255, 2009.

\bibitem[Dhariwal and Nichol(2021)]{dhariwal2021diffusion}
Prafulla Dhariwal and Alexander Nichol.
\newblock Diffusion models beat gans on image synthesis.
\newblock \emph{Advances in neural information processing systems}, 34:\penalty0 8780--8794, 2021.

\bibitem[Goodfellow et~al.(2014)Goodfellow, Pouget-Abadie, Mirza, Xu, Warde-Farley, Ozair, Courville, and Bengio]{GAN}
Ian~J. Goodfellow, Jean Pouget-Abadie, Mehdi Mirza, Bing Xu, David Warde-Farley, Sherjil Ozair, Aaron Courville, and Yoshua Bengio.
\newblock Generative adversarial networks, 2014.

\bibitem[Guo et~al.(2021)Guo, Yang, and Huang]{guo2021image}
Xiefan Guo, Hongyu Yang, and Di Huang.
\newblock Image inpainting via conditional texture and structure dual generation.
\newblock In \emph{Proceedings of the IEEE/CVF international conference on computer vision}, pages 14134--14143, 2021.

\bibitem[Guo et~al.(2024)Guo, Gao, Lu, Zhu, Liu, and He]{onerestore}
Yu Guo, Yuan Gao, Yuxu Lu, Huilin Zhu, Ryan~Wen Liu, and Shengfeng He.
\newblock Onerestore: A universal restoration framework for composite degradation.
\newblock In \emph{European conference on computer vision}, pages 255--272. Springer, 2024.

\bibitem[Heusel et~al.(2017)Heusel, Ramsauer, Unterthiner, Nessler, and Hochreiter]{fid}
Martin Heusel, Hubert Ramsauer, Thomas Unterthiner, Bernhard Nessler, and Sepp Hochreiter.
\newblock Gans trained by a two time-scale update rule converge to a local nash equilibrium.
\newblock \emph{Advances in neural information processing systems}, 30, 2017.

\bibitem[Ho et~al.(2020)Ho, Jain, and Abbeel]{NEURIPS2020_DDPM}
Jonathan Ho, Ajay Jain, and Pieter Abbeel.
\newblock Denoising diffusion probabilistic models.
\newblock In \emph{Advances in Neural Information Processing Systems}, pages 6840--6851. Curran Associates, Inc., 2020.

\bibitem[Huang et~al.(2015)Huang, Singh, and Ahuja]{urban100}
Jia-Bin Huang, Abhishek Singh, and Narendra Ahuja.
\newblock Single image super-resolution from transformed self-exemplars.
\newblock In \emph{Proceedings of the IEEE Conference on Computer Vision and Pattern Recognition (CVPR)}, 2015.

\bibitem[Karras et~al.(2018)Karras, Aila, Laine, and Lehtinen]{celebahq}
Tero Karras, Timo Aila, Samuli Laine, and Jaakko Lehtinen.
\newblock Progressive growing of {GAN}s for improved quality, stability, and variation.
\newblock In \emph{International Conference on Learning Representations}, 2018.

\bibitem[Karras et~al.(2019)Karras, Laine, and Aila]{ffhq}
Tero Karras, Samuli Laine, and Timo Aila.
\newblock A style-based generator architecture for generative adversarial networks.
\newblock In \emph{Proceedings of the IEEE/CVF conference on computer vision and pattern recognition}, pages 4401--4410, 2019.

\bibitem[Kwan et~al.(1996)Kwan, Evans, and Pike]{brainweb1}
Remi K-S Kwan, Alan~C Evans, and G~Bruce Pike.
\newblock An extensible mri simulator for post-processing evaluation.
\newblock In \emph{International conference on visualization in biomedical computing}, pages 135--140. Springer, 1996.

\bibitem[Kwan et~al.(1999)Kwan, Evans, and Pike]{brainweb3}
RK-S Kwan, Alan~C Evans, and G~Bruce Pike.
\newblock Mri simulation-based evaluation of image-processing and classification methods.
\newblock \emph{IEEE transactions on medical imaging}, 18\penalty0 (11):\penalty0 1085--1097, 1999.

\bibitem[Lai et~al.(2025)Lai, Song, Kim, Mitsufuji, and Ermon]{lai2025principlesdiffusionmodels}
Chieh-Hsin Lai, Yang Song, Dongjun Kim, Yuki Mitsufuji, and Stefano Ermon.
\newblock The principles of diffusion models, 2025.

\bibitem[Ledig et~al.(2017)Ledig, Theis, Huszar, Caballero, Cunningham, Acosta, Aitken, Tejani, Totz, Wang, and Shi]{SRGan}
Christian Ledig, Lucas Theis, Ferenc Huszar, Jose Caballero, Andrew Cunningham, Alejandro Acosta, Andrew Aitken, Alykhan Tejani, Johannes Totz, Zehan Wang, and Wenzhe Shi.
\newblock Photo-realistic single image super-resolution using a generative adversarial network, 2017.

\bibitem[Lee et~al.(2013)Lee, Lee, and Kim]{lee2013contrast}
Chulwoo Lee, Chul Lee, and Chang-Su Kim.
\newblock Contrast enhancement based on layered difference representation of 2d histograms.
\newblock \emph{IEEE transactions on image processing}, 22\penalty0 (12):\penalty0 5372--5384, 2013.

\bibitem[Li et~al.(2019)Li, Ren, Fu, Tao, Feng, Zeng, and Wang]{SOTS}
Boyi Li, Wenqi Ren, Dengpan Fu, Dacheng Tao, Dan Feng, Wenjun Zeng, and Zhangyang Wang.
\newblock Benchmarking single-image dehazing and beyond.
\newblock \emph{IEEE Transactions on Image Processing}, 28\penalty0 (1):\penalty0 492--505, 2019.

\bibitem[Li et~al.(2022{\natexlab{a}})Li, Liu, Hu, Wu, Lv, and Peng]{airnet}
Boyun Li, Xiao Liu, Peng Hu, Zhongqin Wu, Jiancheng Lv, and Xi Peng.
\newblock All-in-one image restoration for unknown corruption.
\newblock In \emph{Proceedings of the IEEE/CVF conference on computer vision and pattern recognition}, pages 17452--17462, 2022{\natexlab{a}}.

\bibitem[Li et~al.(2025)Li, Chen, Dong, Tang, and Pan]{li2024foundir}
Hao Li, Xiang Chen, Jiangxin Dong, Jinhui Tang, and Jinshan Pan.
\newblock Foundir: Unleashing million-scale training data to advance foundation models for image restoration.
\newblock In \emph{ICCV}, 2025.

\bibitem[Li et~al.(2022{\natexlab{b}})Li, Guo, Lin, Li, Feng, and Wang]{li2022misf}
Xiaoguang Li, Qing Guo, Di Lin, Ping Li, Wei Feng, and Song Wang.
\newblock Misf: Multi-level interactive siamese filtering for high-fidelity image inpainting.
\newblock In \emph{Proceedings of the IEEE/CVF conference on computer vision and pattern recognition}, pages 1869--1878, 2022{\natexlab{b}}.

\bibitem[Liang et~al.(2021)Liang, Cao, Sun, Zhang, Van~Gool, and Timofte]{liang2021swinir}
Jingyun Liang, Jiezhang Cao, Guolei Sun, Kai Zhang, Luc Van~Gool, and Radu Timofte.
\newblock Swinir: Image restoration using swin transformer.
\newblock In \emph{Proceedings of the IEEE/CVF international conference on computer vision}, pages 1833--1844, 2021.

\bibitem[Lipman et~al.(2022)Lipman, Chen, Ben-Hamu, Nickel, and Le]{lipman2022flow}
Yaron Lipman, Ricky~TQ Chen, Heli Ben-Hamu, Maximilian Nickel, and Matt Le.
\newblock Flow matching for generative modeling.
\newblock \emph{arXiv preprint arXiv:2210.02747}, 2022.

\bibitem[Liu et~al.(2023{\natexlab{a}})Liu, Vahdat, Huang, Theodorou, Nie, and Anandkumar]{I2SB}
Guan-Horng Liu, Arash Vahdat, De-An Huang, Evangelos~A. Theodorou, Weili Nie, and Anima Anandkumar.
\newblock I$^2$sb: Image-to-image schr\"odinger bridge, 2023{\natexlab{a}}.

\bibitem[Liu et~al.(2021)Liu, Xu, Yang, Fan, and Huang]{VELOLL}
Jiaying Liu, Dejia Xu, Wenhan Yang, Minhao Fan, and Haofeng Huang.
\newblock Benchmarking low-light image enhancement and beyond.
\newblock \emph{International Journal of Computer Vision}, 129\penalty0 (4):\penalty0 1153--1184, 2021.

\bibitem[Liu et~al.(2024)Liu, Wang, Fan, Wang, Tang, and Qu]{liu2024rddm}
Jiawei Liu, Qiang Wang, Huijie Fan, Yinong Wang, Yandong Tang, and Liangqiong Qu.
\newblock Residual denoising diffusion models, 2024.

\bibitem[Liu et~al.(2025)Liu, Liao, Chen, Xiao, Wang, Lin, and Satoh]{liu2025transref}
Taorong Liu, Liang Liao, Delin Chen, Jing Xiao, Zheng Wang, Chia-Wen Lin, and Shin’ichi Satoh.
\newblock Transref: Multi-scale reference embedding transformer for reference-guided image inpainting.
\newblock \emph{Neurocomputing}, 632:\penalty0 129749, 2025.

\bibitem[Liu et~al.(2023{\natexlab{b}})Liu, Gong, and Liu]{liu2023flow}
Xingchao Liu, Chengyue Gong, and Qiang Liu.
\newblock Flow straight and fast: Learning to generate and transfer data with rectified flow.
\newblock In \emph{Proc. ICLR}, 2023{\natexlab{b}}.

\bibitem[Lugmayr et~al.(2022)Lugmayr, Danelljan, Romero, Yu, Timofte, and Gool]{lugmayr2022repaintinpaintingusingdenoising}
Andreas Lugmayr, Martin Danelljan, Andres Romero, Fisher Yu, Radu Timofte, and Luc~Van Gool.
\newblock Repaint: Inpainting using denoising diffusion probabilistic models, 2022.

\bibitem[Luo et~al.(2023)Luo, Gustafsson, Zhao, Sjölund, and Schön]{luo2023IRSDE}
Ziwei Luo, Fredrik~K. Gustafsson, Zheng Zhao, Jens Sjölund, and Thomas~B. Schön.
\newblock Image restoration with mean-reverting stochastic differential equations, 2023.

\bibitem[Luo et~al.(2024)Luo, Gustafsson, Zhao, Sjölund, and Schön]{luo2024daclip}
Ziwei Luo, Fredrik~K. Gustafsson, Zheng Zhao, Jens Sjölund, and Thomas~B. Schön.
\newblock Controlling vision-language models for multi-task image restoration, 2024.

\bibitem[Ma et~al.(2015)Ma, Zeng, and Wang]{ma2015perceptual}
Kede Ma, Kai Zeng, and Zhou Wang.
\newblock Perceptual quality assessment for multi-exposure image fusion.
\newblock \emph{IEEE Transactions on Image Processing}, 24\penalty0 (11):\penalty0 3345--3356, 2015.

\bibitem[Ma et~al.(2017)Ma, Duanmu, Wu, Wang, Yong, Li, and Zhang]{Ma2017WED}
Kede Ma, Zhengfang Duanmu, Qingbo Wu, Zhou Wang, Hongwei Yong, Hongliang Li, and Lei Zhang.
\newblock Waterloo exploration database: New challenges for image quality assessment models.
\newblock \emph{IEEE Transactions on Image Processing}, 26\penalty0 (2):\penalty0 1004--1016, 2017.

\bibitem[Meng et~al.(2022)Meng, He, Song, Song, Wu, Zhu, and Ermon]{meng2022sdeditguidedimagesynthesis}
Chenlin Meng, Yutong He, Yang Song, Jiaming Song, Jiajun Wu, Jun-Yan Zhu, and Stefano Ermon.
\newblock Sdedit: Guided image synthesis and editing with stochastic differential equations, 2022.

\bibitem[Mescheder et~al.(2018)Mescheder, Geiger, and Nowozin]{mescheder2018trainingmethodsgansactually}
Lars Mescheder, Andreas Geiger, and Sebastian Nowozin.
\newblock Which training methods for gans do actually converge?, 2018.

\bibitem[Nah et~al.(2017)Nah, Hyun~Kim, and Mu~Lee]{gopro}
Seungjun Nah, Tae Hyun~Kim, and Kyoung Mu~Lee.
\newblock Deep multi-scale convolutional neural network for dynamic scene deblurring.
\newblock In \emph{Proceedings of the IEEE conference on computer vision and pattern recognition}, pages 3883--3891, 2017.

\bibitem[Neumann et~al.(2019)Neumann, Pinto, Zhai, and Houlsby]{ucmerced}
Maxim Neumann, Andre~Susano Pinto, Xiaohua Zhai, and Neil Houlsby.
\newblock In-domain representation learning for remote sensing, 2019.

\bibitem[Ning et~al.(2025)Ning, Li, Su, Jia, Liu, Beneš, Chen, Salah, and Ertugrul]{dcfdiff2025}
Mang Ning, Mingxiao Li, Jianlin Su, Haozhe Jia, Lanmiao Liu, Martin Beneš, Wenshuo Chen, Albert~Ali Salah, and Itir~Onal Ertugrul.
\newblock Dctdiff: Intriguing properties of image generative modeling in the dct space, 2025.

\bibitem[Pang et~al.(2021)Pang, Lin, Qin, and Chen]{pang2021imagetoimagetranslationmethodsapplications}
Yingxue Pang, Jianxin Lin, Tao Qin, and Zhibo Chen.
\newblock Image-to-image translation: Methods and applications, 2021.

\bibitem[Potlapalli et~al.(2023)Potlapalli, Zamir, Khan, and Khan]{potlapalli2023promptir}
Vaishnav Potlapalli, Syed~Waqas Zamir, Salman Khan, and Fahad Khan.
\newblock Promptir: Prompting for all-in-one image restoration.
\newblock In \emph{Thirty-seventh Conference on Neural Information Processing Systems}, 2023.

\bibitem[Qian et~al.(2018)Qian, Tan, Yang, Su, and Liu]{rain100}
Rui Qian, Robby~T Tan, Wenhan Yang, Jiajun Su, and Jiaying Liu.
\newblock Attentive generative adversarial network for raindrop removal from a single image.
\newblock In \emph{Proceedings of the IEEE conference on computer vision and pattern recognition}, pages 2482--2491, 2018.

\bibitem[Rim et~al.(2020)Rim, Lee, Won, and Cho]{rim2020real}
Jaesung Rim, Haeyun Lee, Jucheol Won, and Sunghyun Cho.
\newblock Real-world blur dataset for learning and benchmarking deblurring algorithms.
\newblock In \emph{European conference on computer vision}, pages 184--201. Springer, 2020.

\bibitem[Ronneberger et~al.(2015)Ronneberger, Fischer, and Brox]{unet}
Olaf Ronneberger, Philipp Fischer, and Thomas Brox.
\newblock U-net: Convolutional networks for biomedical image segmentation.
\newblock In \emph{International Conference on Medical image computing and computer-assisted intervention}, pages 234--241. Springer, 2015.

\bibitem[Saharia et~al.(2021)Saharia, Ho, Chan, Salimans, Fleet, and Norouzi]{SR3}
Chitwan Saharia, Jonathan Ho, William Chan, Tim Salimans, David~J. Fleet, and Mohammad Norouzi.
\newblock Image super-resolution via iterative refinement, 2021.

\bibitem[Saharia et~al.(2022)Saharia, Chan, Chang, Lee, Ho, Salimans, Fleet, and Norouzi]{saharia2022palette}
Chitwan Saharia, William Chan, Huiwen Chang, Chris Lee, Jonathan Ho, Tim Salimans, David Fleet, and Mohammad Norouzi.
\newblock Palette: Image-to-image diffusion models.
\newblock In \emph{ACM SIGGRAPH 2022 conference proceedings}, pages 1--10, 2022.

\bibitem[Song et~al.(2020)Song, Meng, and Ermon]{song2022ddim}
Jiaming Song, Chenlin Meng, and Stefano Ermon.
\newblock Denoising diffusion implicit models.
\newblock \emph{arXiv preprint arXiv:2010.02502}, 2020.

\bibitem[Song and Ermon(2019)]{song2019generative}
Yang Song and Stefano Ermon.
\newblock Generative modeling by estimating gradients of the data distribution.
\newblock \emph{Advances in neural information processing systems}, 32, 2019.

\bibitem[Song et~al.(2021)Song, Sohl-Dickstein, Kingma, Kumar, Ermon, and Poole]{song2021scorebased}
Yang Song, Jascha Sohl-Dickstein, Diederik~P Kingma, Abhishek Kumar, Stefano Ermon, and Ben Poole.
\newblock Score-based generative modeling through stochastic differential equations.
\newblock In \emph{International Conference on Learning Representations}, 2021.

\bibitem[Tian et~al.(2025)Tian, Liao, Liu, Li, and Ren]{tian2025dfpir}
Xiangpeng Tian, Xiangyu Liao, Xiao Liu, Meng Li, and Chao Ren.
\newblock Degradation-aware feature perturbation for all-in-one image restoration.
\newblock In \emph{Proceedings of the Computer Vision and Pattern Recognition Conference}, pages 28165--28175, 2025.

\bibitem[Wang et~al.(2025)Wang, Zhang, Guo, Wang, Ma, and Du]{wang2025dgsolver}
Hebaixu Wang, Jing Zhang, Haonan Guo, Di Wang, Jiayi Ma, and Bo Du.
\newblock Dgsolver: Diffusion generalist solver with universal posterior sampling for image restoration.
\newblock \emph{arXiv preprint arXiv:2504.21487}, 2025.

\bibitem[Wang et~al.(2013)Wang, Zheng, Hu, and Li]{wang2013naturalness}
Shuhang Wang, Jin Zheng, Hai-Miao Hu, and Bo Li.
\newblock Naturalness preserved enhancement algorithm for non-uniform illumination images.
\newblock \emph{IEEE transactions on image processing}, 22\penalty0 (9):\penalty0 3538--3548, 2013.

\bibitem[Wang et~al.(2004)Wang, Bovik, Sheikh, and Simoncelli]{ssim}
Zhou Wang, Alan~C Bovik, Hamid~R Sheikh, and Eero~P Simoncelli.
\newblock Image quality assessment: from error visibility to structural similarity.
\newblock \emph{IEEE transactions on image processing}, 13\penalty0 (4):\penalty0 600--612, 2004.

\bibitem[Wei et~al.(2018)Wei, Wang, Yang, and Liu]{lol}
Chen Wei, Wenjing Wang, Wenhan Yang, and Jiaying Liu.
\newblock Deep retinex decomposition for low-light enhancement.
\newblock \emph{arXiv preprint arXiv:1808.04560}, 2018.

\bibitem[Whang et~al.(2022)Whang, Delbracio, Talebi, Saharia, Dimakis, and Milanfar]{whang2022deblurring}
Jay Whang, Mauricio Delbracio, Hossein Talebi, Chitwan Saharia, Alexandros~G Dimakis, and Peyman Milanfar.
\newblock Deblurring via stochastic refinement.
\newblock In \emph{Proc. CVPR}, pages 16293--16303, 2022.

\bibitem[Xie et~al.(2025)Xie, Liu, Lin, Fan, Han, Tang, and Qu]{xie2025arra}
Xing Xie, Jiawei Liu, Ziyue Lin, Huijie Fan, Zhi Han, Yandong Tang, and Liangqiong Qu.
\newblock Unleashing the potential of large language models for text-to-image generation through autoregressive representation alignment, 2025.

\bibitem[Yang et~al.(2017)Yang, Tan, Feng, Liu, Guo, and Yan]{yang2017deep}
Wenhan Yang, Robby~T Tan, Jiashi Feng, Jiaying Liu, Zongming Guo, and Shuicheng Yan.
\newblock Deep joint rain detection and removal from a single image.
\newblock In \emph{Proceedings of the IEEE conference on computer vision and pattern recognition}, pages 1357--1366, 2017.

\bibitem[Zamfir et~al.(2025)Zamfir, Wu, Mehta, Tan, Paudel, Zhang, and Timofte]{MOCEIRS}
Eduard Zamfir, Zongwei Wu, Nancy Mehta, Yuedong Tan, Danda~Pani Paudel, Yulun Zhang, and Radu Timofte.
\newblock Complexity experts are task-discriminative learners for any image restoration.
\newblock In \emph{Proceedings of the Computer Vision and Pattern Recognition Conference}, pages 12753--12763, 2025.

\bibitem[Zamir et~al.(2022)Zamir, Arora, Khan, Hayat, Khan, and Yang]{Zamir2021Restormer}
Syed~Waqas Zamir, Aditya Arora, Salman Khan, Munawar Hayat, Fahad~Shahbaz Khan, and Ming-Hsuan Yang.
\newblock Restormer: Efficient transformer for high-resolution image restoration.
\newblock In \emph{CVPR}, 2022.

\bibitem[Zeng et~al.(2025)Zeng, Wang, Chen, Su, and Liu]{vlunet}
Haijin Zeng, Xiangming Wang, Yongyong Chen, Jingyong Su, and Jie Liu.
\newblock Vision-language gradient descent-driven all-in-one deep unfolding networks.
\newblock In \emph{Proceedings of the Computer Vision and Pattern Recognition Conference}, pages 7524--7533, 2025.

\bibitem[Zhang et~al.(2018)Zhang, Isola, Efros, Shechtman, and Wang]{lpips}
Richard Zhang, Phillip Isola, Alexei~A Efros, Eli Shechtman, and Oliver Wang.
\newblock The unreasonable effectiveness of deep features as a perceptual metric.
\newblock In \emph{Proceedings of the IEEE conference on computer vision and pattern recognition}, pages 586--595, 2018.

\bibitem[Zheng et~al.(2024)Zheng, Wu, Yang, Zhang, Hu, and Zheng]{zheng2024diffuir}
Dian Zheng, Xiao-Ming Wu, Shuzhou Yang, Jian Zhang, Jian-Fang Hu, and Wei-Shi Zheng.
\newblock Selective hourglass mapping for universal image restoration based on diffusion model.
\newblock In \emph{Proceedings of the IEEE/CVF conference on computer vision and pattern recognition}, pages 25445--25455, 2024.

\bibitem[Zheng et~al.(2025)Zheng, Lin, Guo, Zhou, Wang, and Qu]{zheng2025fedvlmbench}
Weiying Zheng, Ziyue Lin, Pengxin Guo, Yuyin Zhou, Feifei Wang, and Liangqiong Qu.
\newblock Fedvlmbench: Benchmarking federated fine-tuning of vision-language models, 2025.

\bibitem[Zhu et~al.(2020)Zhu, Park, Isola, and Efros]{cyclegan}
Jun-Yan Zhu, Taesung Park, Phillip Isola, and Alexei~A. Efros.
\newblock Unpaired image-to-image translation using cycle-consistent adversarial networks, 2020.

\bibitem[Zhu et~al.(2023)Zhu, Wang, Fu, Yang, Guo, Dai, Qiao, and Hu]{wgwsnet}
Yurui Zhu, Tianyu Wang, Xueyang Fu, Xuanyu Yang, Xin Guo, Jifeng Dai, Yu Qiao, and Xiaowei Hu.
\newblock Learning weather-general and weather-specific features for image restoration under multiple adverse weather conditions.
\newblock In \emph{Proceedings of the IEEE/CVF conference on computer vision and pattern recognition}, pages 21747--21758, 2023.

\bibitem[Özdenizci and Legenstein(2023)]{weatherdiff}
Ozan Özdenizci and Robert Legenstein.
\newblock Restoring vision in adverse weather conditions with patch-based denoising diffusion models.
\newblock \emph{IEEE Transactions on Pattern Analysis and Machine Intelligence}, 45\penalty0 (8):\penalty0 10346--10357, 2023.

\end{thebibliography}
}

\clearpage
\setcounter{page}{1}
\maketitlesupplementary

\appendix

\section{Derivations and Proofs}

\subsection{Proofs of Proposition 3.1}
In this section, we give a proof to the aforementioned \textbf{proposition 3.1}.
\label{Appendix:proposition}
\renewcommand{\theproposition}{3.1}
 \begin{proposition}
    Let \( P \) and \( Q \) be two distinct probability distributions over a space \( \mathcal{X} \). Suppose that we inject Gaussian noise \( \mathcal{N}(0, \sigma^2) \) (with \( \sigma \neq 0 \)) to both distributions and denote \( P_\sigma \) and \( Q_\sigma \) as the resulting distributions. Then, the Kullback-Leibler (KL) divergence between \( P_\sigma \) and \( Q_\sigma \) is less than the KL divergence between \( P \) and \( Q \):
    \begin{align}
    D_{\text{KL}}(P_\sigma \parallel Q_\sigma) < D_{\text{KL}}(P \parallel Q)
\end{align}
\end{proposition}
\noindent \textbf{Proof:}
Let \( P \) and \( Q \) be two distinct probability distributions (\( P \neq Q \)) with corresponding probability density functions \( p(x) \) and \( q(x) \). The KL divergence is defined as:
\begin{align}
D_{KL}(P \parallel Q) = \int p(x) \log \left( \frac{p(x)}{q(x)} \right) dx
\end{align}
\noindent where the integral is taken over the entire support of \(x\).
To add Gaussian noise to each data distribution in a random manner, we consider a Gaussian kernel \( K(y|x) \), which represents the probability of output \(y\) given input \(x\). The Gaussian kernel is defined as:
\begin{align}
K(y|x) = \frac{1}{\sqrt{2 \pi \sigma^2}} \exp \left( -\frac{(y - x)^2}{2\sigma^2} \right),
\end{align}
where \( \sigma^2 \) is the variance of the noise. Adding Gaussian noise, the distributions for \( P \) and \( Q \) are modified to \( P_\sigma \) and \( Q_\sigma \), with the corresponding probability densities:
\begin{align}
p_\sigma(y) = \int K(y|x) p(x) dx, 
\end{align}
\vspace{-1.5em}
\begin{align}
q_\sigma(y) = \int K(y|x) q(x) dx.
\end{align}
According to the definition:
\begin{align}
D_{KL}(P_\sigma \parallel Q_\sigma) 
&= \int p_\sigma(y) \log \left( \frac{p_\sigma(y)}{q_\sigma(y)} \right) dy \notag \\
&= D_{KL}(P(Y) \parallel Q(Y))
\label{Eq:17}
\end{align}
\noindent Define the joint density \(p(x, y) = p(x) K(y|x)\) and \(q(x, y) = q(x) K(y|x)\). Note that due to Gaussian noise being independent, \( K(y|x) \) does not depend on \(p\) or \(q\), so \( p(x|y) = K(y|x) \) for both \( P \) and \( Q \).
Here, KL divergence \( D_{KL}(P \parallel Q) \) is related to \(x\), while \( D_{KL}(P_\sigma \parallel Q_\sigma) \) is related to \(y\). We can now write the joint KL divergence as:
\begin{align}
D_{KL}(P(X, Y) \parallel Q(X, Y)) = \int p(x, y) \log \left( \frac{p(x, y)}{q(x, y)} \right) dx dy
\end{align} \vspace{-1.5em}
Since
\begin{align}
\frac{p(x, y)}{q(x, y)} = \frac{p(x) K(y|x)}{q(x) K(y|x)} = \frac{p(x)}{q(x)}
\label{Eq:19}
\end{align}
\vspace{-1.5em}
\begin{align}
&\!\!\!\!\!\!\!\!\!\!D_{KL}(P(X, Y) \parallel Q(X, Y)) \notag \\
&= \iint p(x) K(y|x) \log \left( \frac{p(x)}{q(x)} \right) dx dy \quad \text{(Eq.~\ref{Eq:19})} \notag \\
&= \int p(x) \log \left( \frac{p(x)}{q(x)} \right) dx \quad \text{($\int K(y|x) dy = 1$)} \notag \\
&= D_{KL}(P \parallel Q)
\label{Eq:20}
\end{align}
The joint KL divergence can be decomposed as:
\begin{align}
&D_{KL}(P(X, Y) \parallel Q(X, Y)) =D_{KL}(P(Y) \parallel Q(Y)) \notag \\
&+ D_{KL}(P(X|Y) \parallel Q(X|Y) | P(Y))
\end{align}
where \( D_{KL}(P(Y) \parallel Q(Y)) \) is the KL divergence of the marginals \( P_\sigma \) and \( Q_\sigma \), i.e., \( D_{KL}(P_\sigma \parallel Q_\sigma) \). The second term \( D_{KL}(P(X|Y) \parallel Q(X|Y) | P(Y)) \) is positive (\( P \neq Q \)), hence
\vspace{-0.5em}
\begin{align}
\!\!\!\!\!D_{KL}(P(X, Y) \parallel Q(X, Y)) < D_{KL}(P(Y) \parallel Q(Y))
\label{Eq:22}
\end{align}
\vspace{-1em}
According to Eq.~\ref{Eq:17} and Eq.~\ref{Eq:20}, we transfer Eq.~\ref{Eq:22} to:
\begin{align}
D_{KL}(P_\sigma \parallel Q_\sigma) < D_{KL}(P \parallel Q)
\end{align}
This means that, for most cases, if \( P \neq Q \), adding Gaussian noise with $\sigma \neq 0$ will decrease the KL divergence between P and Q.

\subsection{Proofs in Our Method}
In this section, we give a detailed explanation to section 3.2, including detailed explanation of forward sampling process and proof of reverse sampling process.
\label{Appendix:DRDD}
\paragraph{Reverse Sampling steps of Residual-Removal Stage.}
Given Eq.~\ref{eq:3} and Eq.~\ref{eq:5}, we have:


\begin{align}
I_{t-1}^{(2)} &= I_0^\theta + \bar{\alpha}_{t-1} I_{\text{res}}^\theta + \sigma \epsilon_t \notag \\
&= (I_t^{(2)} - \bar{\alpha}_t I_{\text{res}}^\theta) + \bar{\alpha}_{t-1} I_{\text{res}}^\theta + \sigma \epsilon_t \notag  \\
&= I_t^{(2)} - (\bar{\alpha}_t - \bar{\alpha}_{t-1}) I_{\text{res}}^\theta + \sigma \epsilon_t \notag   \\
&= I_t^{(2)} - \bar{\alpha}_tI_{\text{res}}^\theta + \sigma   \epsilon_t \notag    \\
&= I_t^{(2)} - \alpha_tI_{\text{res}}^\theta \notag    \\
\end{align}

Finally, we have
\begin{align}
I_{t-1}^{(2)} &= I_t^{(2)} - \alpha_tI_{\text{res}}^\theta(I_{t}^{(2)}, I_{in}, t).  \qquad (Eq.~\ref{eq:6})
\end{align}


\paragraph{Reverse Sampling steps of Denoising Stage.} Given \begin{align}
I^{(1)}_{t} = I^{(1)}_0 + \bar\beta_{t}\,\varepsilon, \quad (Eq.~\ref{eq:2})
\end{align} and \begin{align}
\scriptsize 
p_\theta\big(I^{(1)}_{t-1} \mid I^{(1)}_{t}\big)
&:= q_\sigma\big(I^{(1)}_{t-1} \mid I^{(1)}_t, I_0^{(1)}(\theta)\big) \nonumber \quad (Eq.~\ref{eq:7}) \\
&\!\!\!\!\!\!\!\!\!\!\!\!\!\!\!\!\!\!\!\!\!\!\!\!\!\!\!\!\!\!\!\!\!\!\!= \mathcal{N} (I_{t-1};I_0^{(1)}(\theta)+ \sqrt{\bar{\beta }_{t-1}^2 -\sigma _t^2} \frac{(I^{(1)}_t - I_0^{(1)}(\theta))}{\bar{\beta }_{t}} ,\sigma _{t}^2\mathbf{I}
 ), 
\label{eq:28}
\end{align}
\vspace{-2em}

\begin{align}
& I_{t-1}^{(1)} \notag \\
&= I_0^{(1)}(\theta) + \sqrt{\bar{\beta}_{t-1}^{2} - \sigma_{t}^{2}}\;\frac{(I_{t}^{(1)}-I_0^{(1)}(\theta))}{\bar{\beta}_{t}} +\sigma_{t}\varepsilon_{t}\ \notag\\
&= (I_t^{1} - \bar{\beta}_t\epsilon_t ) + \sqrt{\bar{\beta}_{t-1}^{2} - \sigma_t^{2}}\;\frac{\,(I_t^{1}-(I_t^{1} - \bar{\beta}_t\epsilon_t ))\,}{\bar{\beta}_t} +\sigma_t\varepsilon_t\   \notag \\
&= (I_{t}^{(1)} - \bar{\beta}_{t} \epsilon_{t})
   + \epsilon_{t} \sqrt{\bar{\beta}_{t-1}^{2} - \sigma_{t}^{2}}
   + \sigma_{t} \varepsilon_{t} \notag \\
&=I_{t-1}^{(1)}= I_{t}^{(1)} -(\bar{\beta }_t-\sqrt{\bar{\beta }_{t-1}^2-\sigma _t^2} )\, 
{\epsilon}_{\theta}(I_{t}^{(1)}, t) +\sigma_t\varepsilon_t.\ 
\end{align}
\noindent where $\sigma _t^2=\eta \beta_t^2\bar{\beta} _{t-1}^2/\bar{\beta} _{t}^2$. When generation process is deterministic ($\eta=0$), we have:
\begin{align}
I_{t-1}^{(1)}
&= I_{t}^{(1)} - (\bar{\beta}_{t} - \bar{\beta}_{t-1})\, 
I_{\epsilon}^{\theta}\!\big(I_{t}^{(1)},\, t\big) 
\end{align}
\paragraph{Derivation of Eq.~\ref{eq:7} (Eq.~\ref{eq:28}).}
Similar to proof in RDDM~\citep{liu2024rddm} A.2, we have:
\begin{align}
   q(I_t|I_0) &= \mathcal{N} (I_{t};I_0, \bar{\beta}_t^2 \mathbf{I}).\label{Eq:31}
\end{align}
Similar to the evolution from DDPM~\citep{NEURIPS2020_DDPM} to DDIM~\citep{song2022ddim}, we can prove the statement with an induction argument for $t$ from $T$ to $1$. Assuming that Eq.~\ref{Eq:31} holds at $T$, we just need to verify $q(I_{t-1}|I_0)$ at $t-1$ from $q(I_t|I_0)$ at $t$ using Eq.~\ref{Eq:31}. Given:
\begin{align}
    & q(I_t|I_0) = \mathcal{N} (I_{t};I_0, \bar{\beta}_t^2 \mathbf{I}),\label{Eq:32}                                                \\
    & q_\sigma (I_{t-1}|I_t,I_0) = \mathcal{N} (I_{t-1}; I_0 + \sqrt{\bar{\beta}_{t-1}^2 - \sigma _t^2} \frac{(I_t - I_0)}{\bar{\beta}_t}, \sigma _t^2 \mathbf{I} ),\label{Eq:33} \\
    & q(I_{t-1}|I_0):=\mathcal{N}(\tilde{\mu}_{t-1}, \tilde{\Sigma}_{t-1})\label{Eq:34}
\end{align}
Similar to obtaining $p(y)$ from $p(x)$ and $p(y|x)$ using Eq.2.113-Eq.2.115 in~\citep{bishop2006pattern}, the values of $ \tilde{\mu}_{t-1}$ and $\tilde{\Sigma}_{t-1}$ are derived as follows:
\begin{align}
   \tilde{\mu}_{t-1} &= I_0 + \sqrt{\bar{\beta}_{t-1}^2 - \sigma _t^2} \frac{(I_0) - (I_0)}{\bar{\beta}_t} = I_0, \\
   \tilde{\Sigma} _{t-1} &= \sigma _t^2 \mathbf{I} + \left( \frac{\sqrt{\bar{\beta}_{t-1}^2 - \sigma _t^2}}{\bar{\beta}_t} \right)^2 \bar{\beta}_t^2 \mathbf{I} = \bar{\beta}_{t-1}^2 \mathbf{I}.
\end{align}
Therefore, $q(I_{t-1}|I_0) = \mathcal{N} (I_{t-1}; I_0, \bar{\beta}_{t-1}^2 \mathbf{I})$.
In fact, the case ($t = T$) already holds, thus Eq.~\ref{Eq:31} holds for all $t$. We can derive Eq.~\ref{eq:7} and Eq.~\ref{eq:28} from Eq.~\ref{Eq:33}.

\subsection{Derivation of Training Objectives}
\label{Appendix:loss}
In this section, we give a proof to the training objectives.
According to Eq.~\ref{eq:5}, we derive the training objective of residual-removal process as follows:
\begin{align} 
&{L_{res}}(\theta) \notag \\
&= D_{KL}\left(q(I_{t-1}^{(2)} \mid I_{t}^{(2)}, I_0^{(2)}(\theta), I_{\text{res}}^\theta)\right) \parallel p_\theta\left(I_{t-1}^{(2)} \mid I_{t}^{(2)}\right) \notag \\
&= \mathbb{E}\left[\left\Vert I_t - \alpha_t I_{res} - \left(I_t - \alpha_t I_{res}^{\theta}  \right)\right\Vert^2\right] \notag \\
&= \mathbb{E} \left[ \left\| I_{\text{res}} -  I_{\text{res}}^\theta(I_t^{(2)},t,I_{in}) \right\|_1 \right]. \quad (Eq.~\ref{eq:10})
\end{align}

\noindent According to Eq.~\ref{eq:7}, we derive the training objective of denoising process as follows:
\begin{align}
&{L_{\epsilon}}(\theta) \notag \\
&= D_{KL}\left(q(I_{t-1}^{(1)} \mid I_{t}^{(1)}, I_0^{(1)}(\theta)\right) \parallel p_\theta\left(I_{t-1}^{(1)} \mid I_{t}^{(1)}\right) \notag \\
&= \mathbb{E}\left[\left\Vert I_t - \frac{\beta^2_t}{\overline{\beta}_t}\epsilon - \left(I_t - \frac{\beta^2_t}{\overline{\beta}_t}\epsilon^{\theta} \right)\right\Vert^2\right] \notag \\
&= \mathbb{E} \left[ \left\| \epsilon - {\epsilon}_{\theta}(I_t^{(1)}, t) \right\|_1 \right]. \quad (Eq.~\ref{eq:10})
\end{align}

\subsection{Derivation of Decoupled SDE} This section introduces decoupling paradigm in SDE-based diffusion models and show the process of generating samples with reverse-time SDEs.
The forward process formula is as follows:
\label{App:sde}
\begin{equation}
	\diff {x} = \theta_t \, (\mu - {x}) \diff t + \sigma_t \diff w,
	\label{equ:sde}
\end{equation}
where $t$ denote the continuous time variable, w is a standard Wiener process, $\mu$ is the state mean, and $\theta_t, \sigma_t$ are time-dependent positive parameters that characterize the speed of the mean-reversion and the stochastic volatility, respectively. 
In \ref{equ:sde}, $\theta_t \, (\mu - {x}) \diff t$ is the drift term, which governs the evolutionary trend of the state, while $\sigma_t \diff w$ denotes the random noise disturbance affecting the state. We then reverse the SDE to derive an image restoration SDE. During the testing phase, only the score $\nabla_{{x}} \log p_t({x})$ needs to be predicted in this formula:

\begin{equation}
    \diff {x} = \big[ \theta_t \, (\mu - {x}) - \sigma_t^2 \, \nabla_{{x}} \log p_t({x}) \big] \diff t + \sigma_t \diff \hat{w}.
    \label{eq:reverse-irsde}
\end{equation}
We decouple the forward process into a two-stage procedure, adding noise and degradation-specific information respectively, as follows:
\begin{equation}
	\diff {x^{(1)}} = \sigma_t \diff w
	\label{equ:sde_noise}
\end{equation}
\begin{equation}
	\diff {x^{(2)}} = \theta_t \, (\mu - {x^{(2)}}) \diff t + \sigma_t \diff w,
	\label{equ:sde_res}
\end{equation}
Finally, we can reverse the SDE by predicting the score in Formula \ref{eq:reverse-irsde} for each of the two stages.  Given an initial state ${x}_0$, for any state ${x}_i$ at discrete time $i > 0$, the optimum residual reversing solution ${x}_{i-1}^{*}$ in (\ref{equ:sde}) is given by:
    \begin{equation}
    \begin{split}
        {x}_{i-1}^{(2)*} &= \frac{1 - e^{-2 \, \bar{\theta}_{i-1}}}{1 - e^{-2 \, \bar{\theta}_i}} e^{-\theta_i^{'}} ({x}^{(2)}_i - \mu) \\[.6em]
        &\quad+ \frac{1 - e^{-2 \, \theta_i^{'}}}{1 - e^{-2 \, \bar{\theta}_i}} e^{-\bar{\theta}_{i-1}} ({x}^{(2)}_0 - \mu) + \mu.
    \end{split}
    \end{equation}
For noise reversing ${x}_{i-1}^{(1)*}$ is given by:
    \begin{equation}
        \begin{split}
            {x}_{i-1}^{(1)*} &= \frac{1 - \mathrm{e}^{-2\bar{{\theta}}_{i-1}}}{1 - \mathrm{e}^{-2\bar{{\theta}}_{i}}} \mathrm{e}^{-{{\theta}}_{i}} ({x}^{(1)}_i - {x}^{(1)}_0) + {x}^{(1)}_0.
        \end{split}
        \label{app-eq:optimal_denoising_trajectory}
    \end{equation}

\subsection{Noise-Level Selection via Dual MMD Distances}
\label{Appendix:Noise-injection}

In this section, we give a detailed exploration of aforementioned Eq.~\ref{eq:12} in Section 4.5. 
\newcommand{\Dist}{\mathrm{Dist}} 
\begin{equation}
A(\sigma) = \Delta(P^\sigma_s, P^\sigma_t), \quad B(\sigma) = \Delta(P^\sigma_s, P_s)
\end{equation}
Here, \(A(\sigma)\) and \(B(\sigma)\) represent the measures of distance (in this case, based on MMD) between the distributions \(P^\sigma_s\) and \(P^\sigma_t\) for \(A(\sigma)\), and \(P^\sigma_s\) and \(P_s\) for \(B(\sigma)\), where:
- \(P^\sigma_s\) is the distribution of the source after adding Gaussian noise \( \mathcal{N}(0, \sigma^2) \) with \(\sigma\) being the noise strength parameter.
- \(P^\sigma_t\) and \(P_s\) are the target and reference distributions used for comparison.

\vspace{-2em}

\begin{align}
\Delta &= \mathbb{E}_{x,x' \sim P} \left[ \exp\left( -\frac{\|x - x'\|^2}{2\sigma^2} \right) \right] \notag \\
&\quad + \mathbb{E}_{y,y' \sim Q} \left[ \exp\left( -\frac{\|y - y'\|^2}{2\sigma^2} \right) \right] \notag \\
&\quad - 2 \mathbb{E}_{x \sim P, y \sim Q} \left[ \exp\left( -\frac{\|x - y\|^2}{2\sigma^2} \right) \right],
\end{align}

In the above equation, \(\Delta\) is the Maximum Mean Discrepancy between two distributions, \(P\) and \(Q\). Here, \(\sigma\) is a hyperparameter controlling the width of the Gaussian kernel, which affects the sensitivity of the kernel to the differences between samples from the distributions. The term \(\|x - x'\|^2\) is the squared Euclidean distance between two samples, and \(P\) and \(Q\) represent the two distributions being compared.
\vspace{-2em}
\begin{align}
\widehat{A}(\sigma) &= \frac{1}{R}\sum_{r=1}^R \widehat{A}_r(\sigma), \qquad
\widehat{B}(\sigma)  = \frac{1}{R}\sum_{r=1}^R \widehat{B}_r(\sigma).
\end{align}
In this equation, \(\widehat{A}(\sigma)\) and \(\widehat{B}(\sigma)\) are the average estimates of \(A(\sigma)\) and \(B(\sigma)\), computed over \(R\) independent samples. Each sample, \(\widehat{A}_r(\sigma)\) and \(\widehat{B}_r(\sigma)\), is calculated for the \(r\)-th sample from the dataset. Here, \(R\) represents the total number of samples used in the averaging process.
\begin{align}
\widetilde{A}(\sigma) &= 
\frac{\widehat{A}(\sigma) - \min_{\tau \in [0, \infty)} \widehat{A}(\tau)}
     {\max_{\tau \in [0, \infty)} \widehat{A}(\tau) - \min_{\tau \in [0, \infty)} \widehat{A}(\tau) + \epsilon}, \\
\widetilde{B}(\sigma) &=
\frac{\widehat{B}(\sigma) - \min_{\tau \in [0, \infty)} \widehat{B}(\tau)}
     {\max_{\tau \in [0, \infty)} \widehat{B}(\tau) - \min_{\tau \in [0, \infty)} \widehat{B}(\tau) + \epsilon}.
\end{align}
The equations above represent the normalized versions of \(\widehat{A}(\sigma)\) and \(\widehat{B}(\sigma)\), denoted as \(\widetilde{A}(\sigma)\) and \(\widetilde{B}(\sigma)\), respectively. The normalization is done to scale the values of \(\widehat{A}(\sigma)\) and \(\widehat{B}(\sigma)\) to a range between 0 and 1. The small constant \(\varepsilon\) is added to the denominator to avoid division by zero and ensure numerical stability.
\begin{align}
J(\sigma;\lambda) \;=\; \lambda\,\widetilde{A}(\sigma) \;+\; \bigl(1-\lambda\bigr)\,\widetilde{B}(\sigma),
\qquad \lambda \in [0,1].
\end{align}
The function \(J(\sigma; \lambda)\) is the weighted trade-off between \(\widetilde{A}(\sigma)\) and \(\widetilde{B}(\sigma)\), controlled by the parameter \(\lambda\). This trade-off allows us to balance the importance of the two terms, with \(\lambda\) taking values in the range \([0,1]\).
\begin{align}
\sigma^{\star}_{J} \;&=\; \arg\min_{\sigma \in [0, \infty)}\; J(\sigma;\lambda).
\end{align}
Finally, the optimal \(\sigma^{\star}_J\) is selected by minimizing the trade-off function \(J(\sigma;\lambda)\) over the range \( \sigma \in [0, \infty) \). The goal is to find the value of \(\sigma\) that minimizes the trade-off between \(\widetilde{A}(\sigma)\) and \(\widetilde{B}(\sigma)\), optimizing the performance based on the specific task.




\section{Experiment Settings and Dataset}
\label{app:more_datasets}
\begin{figure*}[ht]
\centering
\includegraphics[width=0.99\linewidth]{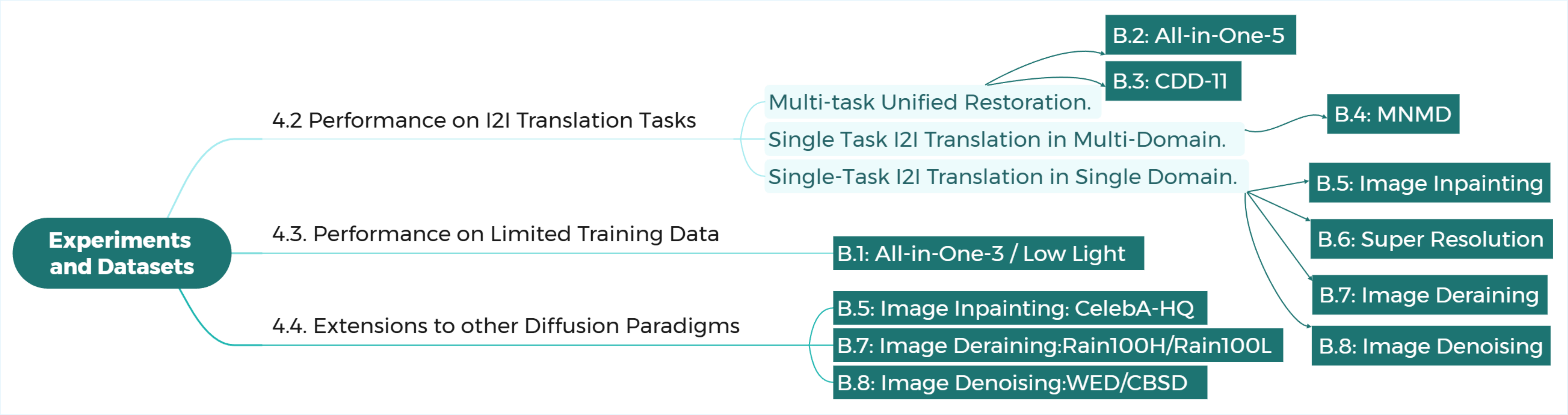}
\caption{\textit{Mindmap of experiments and corresponding datasets.}}
\label{fig:mindmap}
\vspace{-1em}
\end{figure*}
\begin{table}[t]
\centering
\begin{minipage}{0.50\textwidth}
\centering
\caption{\textit{Details for All-in-One image restoration benchmarks.}}
\label{tab:benchmark}
\resizebox{\linewidth}{!}{%
\begin{tabular}{l|ll|ll}
\toprule[0.1em]
\textbf{Task} & \textbf{Training Dataset} & \textbf{Size} & \textbf{Testing Dataset} & \textbf{Size} \\
\midrule[0.05em]



\multirow{3}{*}{All-in-One-3} & WED+BSD400 \citep{bsd400,Ma2017WED} & 5,144 &  CBSD68 \citep{bsd400} & 68\\
  & RESIDE-OTS \citep{SOTS} & 72,135 & SOTS \citep{SOTS} & 492\\
& Rain100L \citep{rain100} & 200 & Rain100L & 100 \\

\midrule[0.05em]
\multirow{5}{*}{All-in-One-5} & GoPro \citep{gopro} & 2,111 & GoPro & 1,111 \\
& WED+BSD400 \citep{bsd400,Ma2017WED} & 5,144 &  CBSD68 \citep{bsd400} & 68\\
   & LoLV1 \citep{lol} & 485 &  LoLV1 & 15\\
  & RESIDE-OTS \citep{SOTS} & 72,135 & SOTS \citep{SOTS} & 492\\
& Rain100L \citep{rain100} & 200 & Rain100L & 100 \\

\midrule[0.05em]
\multirow{3}{*}{MNMD} & UC-Merced \citep{ucmerced}  & 17,010 & UC-Merced & 630 \\
  & WED+BSD400 & 46,296 &  CBSD68 & 204\\
   &  BrainWeb \citep{brainweb1} & 4,689 &  BrainWeb-test & 174\\
\midrule[0.05em]
\multirow{1}{*}{CDD-11} & CDD-11 \citep{onerestore}  & 20,790 & CDD-11 & 2,310 \\
\midrule[0.05em]
\multirow{2}{*}{Low-Light} 
  & LoLV1 \citep{lol} & 485 & LoLV1 & 15\\
& LOL-VE \citep{rain100} & 400 & LOL-VE & 100 \\
\bottomrule[0.1em]
\end{tabular}%
}
\end{minipage}
\end{table}

\noindent\textbf{Experiment Settings.} All experiments are conducted on NVIDIA A6000 48GB GPUs. The network is optimized with L2 loss. Data augmentation includes random horizontal and vertical flips for all tasks and histogram equalization for low-light images. During training, $256\times256$ patches are randomly cropped from augmented images as network inputs. Unless otherwise specified, the batch size is set to 8, the initial learning rate is 8e-5, and the total number of training steps is 300{,}000. 
\\

\label{Appendix:model_architecture}
\noindent\textbf{Model Architecture.} Our residual removal model sets the basic U-Net~\citep{unet} architecture and the hyper-parameters are as follows: $C = 64$, channel multiplier = $(1, 2, 4, 8)$. Our denoising model follows the U-Net architecture in~\citep{dhariwal2021diffusion}, and the parameters are as follows $C = 128$, channel multiplier = $(1, 1, 2, 2, 4, 4)$. 
\\

\label{Appendix:ex setting}
\noindent\textbf{Datasets.} Overall dataset structures are displayed in Fig.~\ref{fig:mindmap}. The detailed dataset introduction are provided later from \ref{app:All-in-One-3} to \ref{Appendix:Denoising}.
\\

\noindent\textbf{Metrics.} To comprehensively evaluate restoration performance, we adopt both distortion-based metrics (PSNR, SSIM \citep{ssim}) and perceptual metrics (LPIPS \citep{lpips}, FID \citep{fid}). Distortion metrics assess pixel-level fidelity, while perceptual metrics compare feature space similarities to better reflect human visual perception. Following standard practice in All-in-One image restoration, all metrics are computed on the RGB channels of full-resolution images.

\subsection{All-in-One-3 \& Low Light}\label{app:All-in-One-3}
In this section, we introduce the datasets we used for experiment ``4.3.Performance on Limited Training Data".
\paragraph{All-in-One-3.} All-in-One-3 is a common setting for unified multiple image restoration tasks training, which contains ``Noise+Haze+Rain".  

\textbf{Image dehazing:}
For image dehazing, we use the outdoor synthesis dataset of RESIDE-OTS~\citep{SOTS} with 72,135 pairs for training and SOTS-Outdoor~\citep{SOTS} dataset for testing with 492 image pairs. These pairs are captured in real-world outdoor scenes and are specifically designed to benchmark the performance of dehazing algorithms under varying weather and lighting conditions, ensuring effective evaluation of dehazing performance across diverse environments.

\textbf{Image deraining:} we use the Rain100L~\citep{rain100} dataset with 200 pairs of images for training and 100 pairs for testing. Detailed introduction to Rain100L is provided in Appendix~\ref{Appendix:Deraining}.

\textbf{Image denoising}: We conduct training using a merged dataset of BSD400~\citep{bsd400} and WED~\citep{Ma2017WED} with 400 and 4,744 clear images, respectively. Noisy images are generated with Gaussian noise ($\sigma$ = (15, 25, 50)). Testing is performed on CBSD68~\citep{bsd400} datasets with 68 samples. Detailed introduction of BSD400~\citep{bsd400} and WED~\citep{Ma2017WED} is provided in Appendix~\ref{Appendix:Denoising}.

\paragraph{Low-Light.}
Following previous image restoration settings~\citep{wang2025dgsolver}, we combine data samples in LOLv1 and VE-LOL-L~\citep{VELOLL} as Low-Light benchmark. 
\textbf{VE-LOL-L}: This dataset is specifically designed to benchmark low-light image enhancement techniques. The VELOLL dataset consists of 400 synthetically paired images for training and 100 paired images for testing, with each pair consisting of a low-light image and its corresponding well-exposed reference image. 
\textbf{LOLv1}: See details in Appendix~\ref{app:lol}.

\subsection{All-in-One-5}
\label{app:AiO5-dataset}
All-in-One-5 contains datasets from the aforementioned three-task setting(All-in-One-3) as well as
additional datasets: GoPro~\citep{gopro} for motion deblurring, and LOLv1~\citep{lol} for lowlight image enhancement. The overall dataset contains ``Noise+Haze+Rain+Light+Blur" in total.

\textbf{Image Deblurring.} The GoPro~\citep{gopro} dataset is a standard benchmark for dynamic scene deblurring, created to enable supervised learning for blind deblurring models. The authors recorded high-frame-rate video and used the high-fps frames as sharp reference images. Instead of convolving with simple uniform kernels, they synthesized realistic, spatially varying motion blur by averaging consecutive sharp frames. The dataset release provides 2,111 blurry-sharp image pairs for training and 1,111 pairs for evaluation.

\textbf{Low Light Enhancement.} 
\label{app:lol}
LOLv1~\citep{lol} is a standard supervised benchmark for low-light enhancement. It consists of paired low-light and well-exposed reference images. The authors collected these pairs by capturing the same scenes under low and normal lighting conditions using a variety of consumer cameras and phones. This process yielded real captures (not purely synthetic) that include realistic sensor noise and color shifts, making the dataset a robust resource. We use 485 image pairs from LoLV1 for training and 15 pairs for testing.

\subsection{CDD-11}
\label{app:CDD11}
The CDD-11 (Composite Degradation Dataset) is a synthetic dataset designed for training and evaluating image restoration models under composite degradation conditions. CDD-11 was introduced in the OneRestore~\citep{onerestore}, from which highlights the importance of dealing with multiple degradation types simultaneously, such as low-light, haze, rain, and snow, rather than addressing each degradation type in isolation. The dataset consists of 11 different degradation conditions, including combinations like ``low\_haze," ``haze\_rain," ``low\_rain," and ``low\_snow," which containing 20,790 image pairs for training and 2,310 image pairs for testing in total. CDD-11 serves as a benchmark for evaluating models that aim to perform restoration under mixed degradation conditions, reflecting real-world scenarios more accurately than datasets with only single degradation types.

\subsection{Multi-Noise and Multi-Domain}
\label{app:MNMD-dataset}
In this section, we introduce a novel denoising benchmark designed for the single task I2I in multi-domains. We name this new benchmark \textbf{M}ulti-\textbf{N}oise and \textbf{M}ulti-\textbf{D}omain (MNMD). 
As shown in Tab.~\ref{tab:benchmark}, MNMD consists of image pairs from various sources, including 46,296 natural images (from WED and BSD400~\citep{rain100}), 17,010 remote sensing images (from UC-Merced~\citep{ucmerced}), and 4,689 medical images (from BrainWeb~\citep{brainweb1, brainweb2, brainweb3}). To better simulate noise-related degradations under different conditions, we introduce three types of noise, each with multiple intensity levels: \textbf{Gaussian noise} is added as $x' = x + n$, where $n \sim \mathcal{N}(0, \sigma^2)$. The parameter $\sigma =  (15, 25, 50) $ controls the standard deviation, reflecting the noise strength; \textbf{Salt-and-Pepper noise} is applied by randomly selecting a fraction $d$ (scale 0.014, 0.039, 0.154) of pixels and replacing each with either the minimum or maximum possible value, thereby simulating random pixel corruption; and \textbf{Poisson noise} is simulated by replacing each pixel with a value drawn from a Poisson distribution whose mean is $x \times \text{peak}$, with $x$ being the original pixel value and $\text{peak}$ (26, 102, 283) acting as a noise scaling factor. 

For evaluation, we use 204 noisy natural images from CBSD68~\citep{bsd400}, as well as 630 and 174 images from UC-Merced and BrainWeb, respectively. The test images are degraded with Gaussian noise ($\sigma$ = 15), salt-and-pepper noise (density = 0.039), and Poisson noise (scale = 102). This setup forms a comprehensive test set for evaluating denoising performance across multiple domains and noise types.

\textbf{The UC‑Merced Land Use Dataset (UC-Merced~\citep{ucmerced}} is a widely‑used remote sensing image dataset designed for land use scene classification, featuring 21 distinct categories of US urban and semi‑urban scenes. Each category contains 100 images, resulting in a total of 2,100 images in the dataset. The images have a resolution of 256×256 pixels and were manually extracted from the United States Geological Survey (USGS) National Map Urban Area Imagery collections for various US urban areas. 
The medical images are stem from \textbf{BrainWeb}~\citep{brainweb1, brainweb2, brainweb3}, which is a website for synthesizing medical images.
Detailed introduction of \textbf{BSD400}~\citep{bsd400} and \textbf{WED}~\citep{Ma2017WED} is provided in Appendix~\ref{Appendix:Denoising}.


\subsection{Image Inpainting}
The CelebA-HQ~\citep{celebahq} (CelebFaces High Quality) dataset is a high-resolution version of the CelebA dataset, specifically designed for training and evaluating face-related image processing tasks, including image inpainting. CelebA-HQ consists of 30,000 high-resolution facial images, each with a resolution of 256x256 pixels. The dataset includes a wide variety of celebrity faces with various attributes such as age, gender, and facial expressions, making it suitable for tasks like face generation and image inpainting. For image inpainting tasks, we adopt the CelebA-HQ dataset for both training and testing, with our training set containing 28,000 image pairs and our test set consisting of 2,000 image pairs. Each image in the dataset is paired with a mask that specifies the region to be inpainted, allowing models to learn to fill in missing parts of the face.

\subsection{Super-Resolution}
For super-resolution, we adopt the FFHQ \citep{ffhq} dataset for training and testing. FFHQ (Flickr-Faces-HQ) dataset is a high-quality facial image dataset consists of 70,000 samples, which is primarily used for computer vision and deep learning research, particularly in applications like facial image generation, editing, and expression recognition. Released by NVIDIA in 2018, the dataset aims to provide a high-resolution standard for facial image generation and recognition tasks. In our experiment, we input 16x16 images and output high-resolution 128x128 images. During training, we employ bicubic interpolation to upsample the 16x16 input images to 128x128, using the upsampled images as input for training.

\subsection{Deraining}
\label{Appendix:Deraining}
For deraining, we adopt the \textbf{Rain100} dataset \cite{yang2017deep} for both training and testing. Rain100 consists of two subsets, Rain100H (Heavy Rain) and Rain100L (Light Rain), each containing paired rainy and ground-truth clean images. The full dataset contains 2,200 image pairs: 300 for Rain100L and 1,900 for Rain100H. For our experiments, we use 100 image pairs from each subset for testing and the remaining pairs for training.


\subsection{Denoising}
\label{Appendix:Denoising}
For image denoising, we leverage two widely used datasets: BSD400~\citep{bsd400} and WED~\citep{Ma2017WED} for training, while CBSD68~\citep{bsd400} and Urban100~\citep{urban100} for testing. Our training set consists of 5,144 image pairs (400 from BSD400 and 4,744 from WED), while our testing set consists of 168 image pairs (68 from CBSD68 and 100 from Urban100).

\textbf{Training datasets.} We use BSD400 and WED to build our training set. BSD400 is a widely used dataset of 400 natural images, usually adopted as a training set in image-restoration research. The WED (Waterloo Exploration Database) is a much larger collection created for image-quality assessment, containing thousands of diverse, high-quality natural images. Combining WED with BSD400 provides a larger and more varied pool of clean images, which is essential for training models that can generalize well to diverse real-world noise. Many recent denoising pipelines~\citep{instructIR, Zamir2021Restormer, potlapalli2023promptir} merge these two datasets to expand their training data, and we follow this paradigm to utilize BSD400+WED as our training set.

\textbf{Testing datasets.} Our models are evaluated on two challenging test sets, CBSD68 and Urban100. CBSD68 is the color version of the BSD68 test set, a benchmark of 68 natural images commonly used to evaluate color-image denoising. Urban100 is a 100-image dataset of high-resolution urban scenes, initially developed for single-image super-resolution but now widely adopted as a challenging benchmark for denoising. Urban scenes often feature strong, repeating geometric patterns (like bricks and windows) that are difficult to reconstruct, making Urban100 a robust test for a model's ability to recover fine structural detail.

\section{Additional Experiments}
\subsection{Implementation Details of Methods in Experiments}
In this section, we introduce the training details of our comparison methods. 

\textbf{DRDD(Ours).} Without specific mentioned, training settings follow experiments setting details in Appendix~\ref{Appendix:ex setting}. In Experiment ``4.3.Performance on Limited Training Data". We start training denoising U-NET from pre-trained parameters, where $C = 256$, channel multiplier = $(1, 1, 2, 2, 4, 4)$. 

\textbf{RDDM.}~\citep{liu2024rddm} All experimental settings are kept consistent with those in the paper for both image deraining and image inpainting tasks.

\textbf{DiffuIR.} It~\citep{zheng2024diffuir} uses a selective hourglass mapping strategy within a diffusion model to handle multiple image restoration tasks with a single, efficient model. All experimental settings are kept consistent with those in the paper.

\textbf{AdAIR}~\citep{cui2025adair} adaptively restores images suffering from various degradations by mining and modulating frequency components, enabling unified and effective all-in-one image restoration. For training stage, it is conducted with a batch size of 32 in the all-in-one setting and a batch size of 8 in the single-task setting. The model is trained on cropped image patches of size 128 × 128 pixels for 150 epochs, which is approximately equivalent to 300,000 steps in DRDD. 

\textbf{DA-CLIP}~\citep{luo2024daclip} employs an image controller to identify degradation types and adaptively restore images affected by diverse distortions, thereby providing a unified and effective solution for multi-task image restoration. For training stage, we use a batch size of 16 in training for All-in-One-5 task. The model is trained for 300,000 steps on cropped image patches of size 256 × 256 pixels. 

\textbf{DFPIR.}~\citep{tian2025dfpir} It uses a unified model with a degradation-aware feature perturbation mechanism, which introduces channel-wise and attention-wise perturbations to mitigate task interference. We follow the experiment settings in the paper.
\begin{table*}[t]
    \centering
    \caption{Comparison of single task image restoration approaches on deraining and denoising with two metrics (SSIM / LPIPS). Best results are highlighted in \textbf{Bold}.}
    \vspace{-2mm}

    \begin{minipage}[t]{0.32\textwidth}
        \centering
        \caption*{(a) Deraining Results}
        \resizebox{\textwidth}{!}{
        \begin{tabular}{l|c|c}
            \toprule[0.15em]
            \multirow{2}{*}{\textbf{Method}} &
            \textbf{Rain100H} &
            \textbf{Rain100L} \\
            \cmidrule(lr){2-2}\cmidrule(l){3-3}
            & SSIM / LPIPS & SSIM / LPIPS \\
            \midrule[0.1em]
            RDDM* & .8806 / .1170 & .9432 / .0250 \\
            \midrule[0.05em]
            IRSDE* & .9041 / .0470 & .9805 / .0140 \\
            \midrule[0.05em]
            \textbf{DRDD (Ours)} & \textbf{.9375} / \textbf{.0400} & \textbf{.9839} / \textbf{.0180} \\
            \bottomrule[0.1em]
        \end{tabular}
        }
    \end{minipage}
    \hfill
    %
    \begin{minipage}[t]{0.66\textwidth}
        \centering
        \caption*{(b) Denoising on CBSD68 and Kodak24 with \(\sigma\in\{15,25,50\}\)}
        \resizebox{\textwidth}{!}{
        \begin{tabular}{l|ccc|ccc}
            \toprule[0.15em]
            \multirow{2}{*}{\textbf{Method}} &
            \multicolumn{3}{c|}{\textbf{CBSD68}} &
            \multicolumn{3}{c}{\textbf{Kodak24}} \\
            \cmidrule(lr){2-4}\cmidrule(l){5-7}
            & $\sigma{=}15$ & $\sigma{=}25$ & $\sigma{=}50$ & $\sigma{=}15$ & $\sigma{=}25$ & $\sigma{=}50$ \\
            & SSIM / LPIPS & SSIM / LPIPS & SSIM / LPIPS & SSIM / LPIPS & SSIM / LPIPS & SSIM / LPIPS \\
            \midrule[0.1em]
            AdAIR & .9340 / .0534 & .8898 / .0967 & .8003 / .1895 & .9269 / .0712 & .8841 / .1114 & .8036 / .1966 \\
            \midrule[0.05em]
            VLUNet & .9341 / .0568 & .8902 / .0984 & \textbf{.8039} / .1938 & .9270 / .0745 & .8838 / .1146 & \textbf{.8079} / .2040 \\
            \midrule[0.1em]
            \textbf{DRDD (Ours)} & \textbf{.9346} / \textbf{.0502} & \textbf{.8920} / \textbf{.0877} & .8013 / \textbf{.1750} & \textbf{.9274} / \textbf{.0680} & \textbf{.8879} / \textbf{.1032} & .8047 / \textbf{.1880} \\
            \bottomrule[0.1em]
        \end{tabular}
        }
    \end{minipage}

    \vspace{-2mm}
    \label{tab:single_task_two_metrics}
\end{table*}

\textbf{IR-SDE}~\citep{luo2023IRSDE} adaptively restores images suffering from various degradations by modeling the degradation process with mean-reverting stochastic differential equations, enabling unified image restoration.
For training stage, all tasks are trained with a batch size of 8 for a total of 500,000 steps. For the image inpainting task, the model is trained on cropped image patches of size 64 × 64 pixels, whereas for the other tasks, is trained on 128 × 128 pixels.

\textbf{De-IRSDE} is a decoupling version of IRSDE~\citep{luo2023IRSDE}. All experimental settings are kept consistent with those of IRSDE~\citep{luo2023IRSDE}.

\textbf{VLUNET}~\citep{vlunet} adaptively restores images suffering from various degradations by leveraging vision-language models to automatically select degradation-specific transforms, enabling unified and effective all-in-one image restoration within a deep unfolding framework. For training stage, we use a batch size of 8 in training for all tasks. The model is trained for 200 epochs on cropped image patches of size 128 × 128 pixels, which is roughly equivalent to 400,000 steps in DRDD. 

\textbf{TransRef.}~\citep{liu2025transref} 
For training stage, we use a batch size of 8 for training across all tasks. The model is trained on the original image pixels for 400,000 steps. 

\textbf{CTSDG.}~\citep{guo2021image} 
For training stage, we use a batch size of 4 for training across all tasks. The model is trained on cropped image patches of size 128 × 128 pixels for 150,000 steps.

\label{app:add_ex}
\begin{table}[t]
  \centering
  \caption{\textit{Performance comparison of inpainting methods at 256$\times$256 (center), 256$\times$256 (irregular), and 64$\times$64 (center) resolutions.} Best results and the second-best results are highlighted in \textcolor{red}{\textbf{red}} and \textcolor{blue}{blue}.} 
  \vspace{-2mm}
  \small
  \setlength{\tabcolsep}{4pt}
  \resizebox{\linewidth}{!}{%
    \begin{tabular}{lcccccc}
      \toprule[0.15em]
      \multirow{2}{*}{\textbf{Method}} &
      \multicolumn{2}{c}{\textbf{256$\times$256 (Center)}} &
      \multicolumn{2}{c}{\textbf{256$\times$256 (Irregular)}} &
      \multicolumn{2}{c}{\textbf{64$\times$64 (Center)}} \\
      \cmidrule(lr){2-3}\cmidrule(lr){4-5}\cmidrule(lr){6-7}
      & LPIPS$\downarrow$ & FID$\downarrow$
      & LPIPS$\downarrow$ & FID$\downarrow$
      & LPIPS$\downarrow$ & FID$\downarrow$ \\
      \midrule[0.1em]
      RDDM \citep{liu2024rddm}       & 0.0862 & 15.67 & 0.0963 & 8.52 & 0.0568 & 14.75 \\
      \midrule[0.1em]
      CTSDG \citep{guo2021image}     & 0.0798  & 11.64 & 0.0722 & 7.87 & 0.0498 & 15.68 \\
      \midrule[0.1em]
      MISF \citep{li2022misf}     & 0.0764     &  11.72    & 0.0803 & \textcolor{blue}{7.68} & 0.0695 & 20.43 \\
      \midrule[0.1em]
      TransRef \citep{liu2025transref}     & \textcolor{blue}{0.0745}      & \textcolor{red}{\textbf{9.13}}     & \textcolor{blue}{0.0692} & 7.80 & \textcolor{blue}{0.0490} & \textcolor{blue}{10.17} \\
      \midrule[0.1em]
       \textbf{DRDD(Ours)}  & \textcolor{red}{\textbf{0.0542}} & \textcolor{blue}{10.30} & \textcolor{red}{\textbf{0.0528}} & \textcolor{red}{\textbf{7.15}} & \textcolor{red}{\textbf{0.0382}} & \textcolor{red}{\textbf{10.02}} \\
      \bottomrule[0.1em]
    \end{tabular}%
  }
  \vspace{-2.5mm}
  \label{tab:4}
\end{table}

\subsection{More Single I2I Tasks in Single Domain.}
\begin{table}[t]
    \centering
    \renewcommand\arraystretch{1.1}
    \caption{\textit{Performance comparison of RDDM and DRDD on Edges2Handbags and Edges2Shoes dataset.}}
    \resizebox{0.99\columnwidth}{!}{
        \begin{tabular}{c|cccccc}
        \toprule
        \multirow{2}{*}{\textbf{\makecell{Table A2}}}
            & \multicolumn{3}{c}{Edges2Handbags-256$\times$256} & \multicolumn{3}{c}{Edges2Shoes-256$\times$256} \\ 
            \cline{2-7}
            & SSIM$\uparrow$ & FID$\downarrow$ & LPIPS$\downarrow$ & SSIM$\uparrow$ & FID$\downarrow$ & LPIPS$\downarrow$ \\
            \midrule
            RDDM  & 0.645 & 5.72 & 0.256  & 0.652 & 23.57 & 0.178 \\
            \textbf{DRDD} & \textbf{0.723} & \textbf{4.76} & \textbf{0.247} & \textbf{0.782} & \textbf{9.658} & \textbf{0.154}\\
            \bottomrule
    \end{tabular}}
    \label{tab:edges2bags}
\end{table}

\begin{table*}[h]
    \centering
    \scriptsize
    \fboxsep0.75pt
    \setlength\tabcolsep{7pt}
    \caption{\textit{Ablation Studies and Further Investigaitons on All-in-One-5 daytaset.} SSIM ($\uparrow$) and \colorbox{clblue!50}{LPIPS ($\downarrow$)} are reported on the full RGB images. The \textbf{best} performances are highlighted.}
    \vspace{-3mm}
    \label{tab:exp:5deg}
    \begin{tabularx}{\textwidth}{lX*{12}{c}}
    \toprule
     & \multirow{2}{*}{Method} & \multicolumn{2}{c}{\textit{Dehazing}} & \multicolumn{2}{c}{\textit{Deraining}} & \multicolumn{2}{c}{\textit{Denoising}} 
     & \multicolumn{2}{c}{\textit{Deblurring}} & \multicolumn{2}{c}{\textit{Low-Light}} & \multicolumn{2}{c}{\multirow{2}{*}{Average}}  \\
     \cmidrule(lr){3-4} \cmidrule(lr){5-6} \cmidrule(lr){7-8} \cmidrule(lr){9-10} \cmidrule(lr){11-12}
     && \multicolumn{2}{c}{SOTS} & \multicolumn{2}{c}{Rain100L} & \multicolumn{2}{c}{BSD68\textsubscript{$\sigma$=25}} 
     & \multicolumn{2}{c}{GoPro} & \multicolumn{2}{c}{LOLv1} &  \\
     \midrule
    \multirow{5}{*} &
    DRDD w.o Denoising Network & .9690 &\cc{.0149} & .9706 &\cc{.0244} & .8771 & \cc{.1136} & .8241 & \cc{.1923} & .8469 & \cc{.1251} & .8971 & \cc{.0941} \\ 
    
    & DRDD with General Denoising Network & .9652 & \cc{.0145} & .9701 & \cc{.0248} & .8867 & \cc{.1054} & .8712 &\cc{.1748} & .8490 & \cc{.1232} & .9084 & \cc{.0885} \\

    & Entangled Baseline & .9601 &\cc{.0136} & .9727 &\cc{.0231} & .8818 & \cc{.1153} & .8405 & \cc{.1849} & .8283 & \cc{.1555} & .8966 & \cc{.0985} \\
    & DRDD (Ours) & \textbf{.9715} & \cc{\textbf{.0122}} & \textbf{.9777} & \cc{\textbf{.0153}} & .8891 & \cc{\textbf{.0977}} & \textbf{.8806} & \cc{\textbf{.1342}} & \textbf{.8640} & \cc{\textbf{.1033}} &\textbf{.0916} &\cc{\textbf{.0762}}\\
     \bottomrule
    \end{tabularx}
    \label{Tab:9}
    \vspace{-3mm}
\end{table*}

\paragraph{Image Inpainting.} To validate the generalization capability of the proposed framework across different tasks, we conducted inpainting experiments under various settings. Specifically, we evaluated the inpainting performance at different resolutions using two mask patterns (center and irregular) on the CelebA-HQ \citep{celebahq} dataset. It is worth noting that irregular masks more closely reflect the restoration demands of complex occlusions encountered in real-world scenarios, whereas rectangular masks are typically used to simulate cases in which local image information is entirely missing. Furthermore, we conducted experiments at both 64×64 and 256×256 resolutions to further confirm the model’s adaptability to different input scales. 
As shown in Tab.~\ref{tab:4}, our method achieved the best or second-best results across all configurations, demonstrating its superior performance and robustness under diverse task conditions. 
\paragraph{Image Deraining.}
We compare DRDD's performance on image deraining through Rain100H and Rain100L datasets with RDDM~\citep{liu2024rddm} and DiffuIR~\citep{zheng2024diffuir}. Results are shown in Tab.~\ref{tab:single_task_two_metrics}.

\begin{figure}[t]
    \caption{\textit{Visual comparison of RDDM and DRDD on Edges2Handbags dataset.}}
    \centering
    \includegraphics[width=0.99\linewidth]{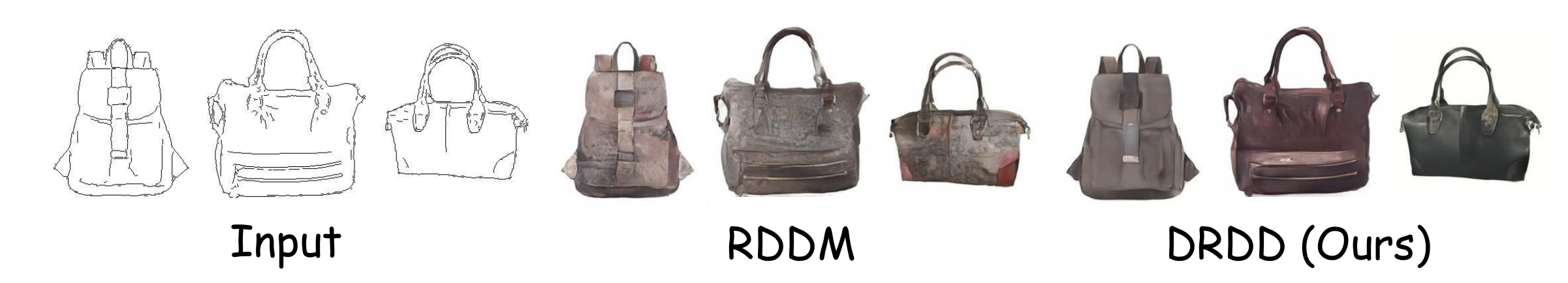} 
    \label{fig:edges2bags}
\end{figure}

\paragraph{Edges to Objects.} As shown in Tab.~\ref{tab:edges2bags} and Fig.~\ref{fig:edges2bags}, we evaluated DRDD on style transfer task using edges2handbags \& edges2shoes datasets. DRDD significantly outperforms its baseline RDDM. Visual results confirm that the injected noise does not adversely affect DRDD.

\paragraph{Image Denoising.}
We compare DRDD's performance on image denoising through BSD400~\citep{bsd400} and WED~\citep{Ma2017WED} with AdaIR\citep{cui2025adair} and VLUNet~\citep{vlunet}. Results are shown in Tab.~\ref{tab:single_task_two_metrics}.

\paragraph{Super Resolution.} To further demonstrate the generalization ability of our method across different tasks, we conducted a qualitative experiment on the FFHQ~\citep{ffhq} dataset to evaluate DRRD's performance in the image super-resolution task.

\subsection{Ablation Studies.}
We compare DRDD against a matched ``coupled" baseline with the same architecture. Although only a single neural network is used, we assign it the combined number of parameters of two decoupled neural networks, along with the double cumulative inference steps. We test it on All-in-One-5 benchmark and results are displayed in Tab.~\ref{Tab:9}.

\subsection{Further Investigations of Denoising Model}
\textbf{Investigating the Training of Denoising Models on Isolated Datasets.}
Based on the properties we previously discussed, the denoising model can be trained on a isolate dataset. To further validate this, We train the denoising model for 300,000 steps on the All-in-One-3 dataset and combined it with the residual removal model, which was trained for 300,000 steps on the All-in-One-5 dataset. The results in Tab.~\ref{Tab:9} show that despite the denoising model never having encountered low-light or deblurring data, it still learned a certain level of generalized denoising capability from the other datasets. 
Such a strategy can significantly reduce the number of training-required parameters.

\textbf{Investigating Inference Without Denoising Models.}
In our decoupled model, the residual removal model is responsible for directional semantic elimination in the noise domain, while the denoising model is tasked with translating the data back to the noise-free domain and performing refined restoration. The necessity for a dedicated denoising model, rather than simply subtracting the noise added in the first step, arises because during directional semantic transformation, the residual removal model also exerts some influence on the noise. This perspective has been confirmed by experiment results in Tab.~\ref{Tab:9}, DRDD without Denoising Network.

\subsection{Cost and Efficiency}
\label{app:cost}
\subsubsection{Parameters, Flops and Inference Time}
In this section, we present the key performance metrics of our model. The floating-point operations (FLOPs) were computed by performing a forward inference using a randomly generated matrix of size 256×256×3. For measuring inference time, the model was trained on the AIOIR5 architecture and evaluated on the Rain100L dataset. The reported inference time was obtained under the condition where the model runs for two steps on each of its two sub-modules: residual removing and denoising.

\begin{table}[htbp]
    \vspace{-0.6em}
    \centering
\caption{\textit{Computational resource comparisons: Parameters, FLOPs, and Runtime.} FLOPs are measured on the patch size of
256 × 256 × 3, while Runtime are measured on Rain100L testing set. The sampling steps of denoising model and residual-removal model are both 2.}
    \renewcommand\arraystretch{1.1}
    \resizebox{\columnwidth}{!}{
    \begin{tabular}{c|l|ccc|c|ccc}
        \hline
        & \textbf{Method} & \textbf{Para (M)} & \textbf{FLOPs (G)} & \textbf{Step} & \textbf{Latency (s)} & \textbf{SSIM$\uparrow$} & \textbf{LPIPS$\downarrow$} & \textbf{FID$\downarrow$} \\
        \hline
        & AdAIR  & 29 & 138 & 1 & 0.24 & \textcolor{blue}{0.909} & 0.089 & 26.1 \\
        & VLUNet  & 123 & 143 & 1 & 0.74 & 0.904 & 0.096 & 27.9 \\
        \hline
        \multirow{6}{*}{\makecell{\textbf{Diff.}}} 
        & DiffUIR & 138 & 284 & 3 & 0.75 & 0.869 & 0.117 & 33.7 \\
        & DA-CLIP & 174 & 380 & 3 & 0.76 & 0.876 & 0.108 & 20.0 \\
        & RDDM - S & 138 & 278 & 4 & 0.92 & 0.878 & 0.105 & 27.8 \\
        & \cellcolor{gray!15}RDDM - L & \cellcolor{gray!15}138$\times$2 & \cellcolor{gray!15}278$\times$2 & \cellcolor{gray!15}4 & \cellcolor{gray!15}1.84 & \cellcolor{gray!15}0.897 & \cellcolor{gray!15}0.098 & \cellcolor{gray!15}23.6 \\
        & DRDD - S & 7+35 & 32+69 & 2 + 2 & 0.10+0.23 & 0.908 & \textcolor{blue}{0.080} & \textcolor{blue}{19.6} \\
        & \cellcolor{gray!15}\textbf{DRDD - L} & \cellcolor{gray!15}7+138 & \cellcolor{gray!15}32+278 & \cellcolor{gray!15}2 + 2 & \cellcolor{gray!15}0.10+0.45 & \cellcolor{gray!15}\textcolor{red}{\textbf{0.916}} & \cellcolor{gray!15}\textcolor{red}{\textbf{0.073}} & \cellcolor{gray!15}\textcolor{red}{\textbf{18.3}} \\
        \hline
    \end{tabular}
    }
    \vspace{-0.6em}
    \label{tab:efficiency_comparison}
\end{table}
        

\begin{table}[h]
    \centering
    \caption{\textit{Performance of different sampling steps.} SSIM ($\uparrow$) and LPIPS ($\downarrow$) are reported.}
    \vspace{-2mm}
    \resizebox{\linewidth}{!}{
    \begin{tabular}{c c c c c c c c}
        \toprule[0.15em]
        \multicolumn{2}{c}{\textbf{Sampling Steps}} &
        \multicolumn{2}{c}{\textbf{Dehazing}} &
        \multicolumn{2}{c}{\textbf{Denoise}} &
        \multicolumn{2}{c}{\textbf{Deraining} }\\
        \cmidrule(lr){1-2}\cmidrule(lr){3-4}\cmidrule(lr){5-6}\cmidrule(lr){7-8}
        de-res &
        de-noise &
        SSIM$\uparrow$ &
        LPIPS$\downarrow$ &
        SSIM$\uparrow$ &
        LPIPS$\downarrow$ &
        SSIM$\uparrow$ &
        LPIPS$\downarrow$ \\
        \midrule[0.1em]
        
        2 & 2 &
        0.9787 & 0.0104 & 0.8919 & 0.0929 & 0.9767 & 0.0149\\
        \midrule[0.1em]
        2 & 5 &
        0.9790 & 0.0103 & 0.8900 & 0.0940 & 0.9774 & 0.0146\\
        \midrule[0.1em]
        5 & 2 &
        0.9791 & 0.0101 & 0.8918 & 0.0930 & 0.9767 & 0.0149\\
        \midrule[0.1em]
        5 & 5 &
        0.9794 & 0.0100 & 0.8899 & 0.0940 & 0.9774 & 0.0146\\
        \midrule[0.1em]
        10 & 10 &
        0.9792 & 0.0101 & 0.8838 & 0.0951 & 0.9777 & 0.0144\\
        \bottomrule[0.1em]
    \end{tabular}
    }
    \vspace{-2mm}
    \label{tab:ablation}
\end{table}

\subsubsection{Influence of sampling steps}
\label{app:sampling_steps}
We conducted experiments to investigate the influence of different sampling steps on the results using the All-in-One-3 dataset. The results, shown in Tab.~\ref{tab:ablation}, indicate that there is no significant difference between 2-step and 10-step inference when using DDIM sampling.

\subsection{Generalization Ability}

\begin{table}[t]
    \centering
    \caption{Performance comparison between our method and other universal image restoration methods on real unseen datasets.  Best results are highlighted in \textcolor{red}{\textbf{red}}, while the second-best results are \textcolor{blue}{\underline{blue}}.}
    \vspace{-2mm}
    \resizebox{\columnwidth}{!}{
    \begin{tabular}{lc c c c}
        \toprule[0.15em]
        \multirow{2}{*}{\textbf{Method}} &
        \textbf{Low-Light} &
        \textbf{Deraining} &
        \textbf{Denoising} &
        \textbf{Deblurring} \\
        \cmidrule(lr){2-2}\cmidrule(lr){3-3}\cmidrule(lr){4-4}\cmidrule(lr){5-5}
        &
        NIQE$\downarrow$ &
        NIQE$\downarrow$ &
        PSNR$\uparrow$ / SSIM$\uparrow$ / LPIPS$\downarrow$ &
        PSNR$\uparrow$ / SSIM$\uparrow$ / LPIPS$\downarrow$ \\
        \midrule[0.1em]

        VLUNet &
        4.40 &
        \textcolor{red}{\textbf{3.73}} &
        24.6 / 0.490 / \textcolor{blue}{\underline{0.664}} &
        26.8 / 0.822 / 0.175 \\
        \midrule[0.1em]
        DiffuIR &
        \textcolor{red}{\textbf{3.89}} &
        4.49 &
        \textcolor{blue}{\underline{28.6}} / \textcolor{blue}{\underline{0.674}} / 0.569 &
        \textcolor{blue}{\underline{28.4}} / \textcolor{blue}{\underline{0.857}} / \textcolor{blue}{\underline{0.186}} \\
        \midrule[0.1em]
        \textbf{DRDD(Ours)} & 
        \textcolor{blue}{\underline{4.31}} &
        \textcolor{blue}{\underline{3.80}} &
        \textcolor{red}{\textbf{34.8}} / \textcolor{red}{\textbf{0.865}} / \textcolor{red}{\textbf{0.274}} &
         \textcolor{red}{\textbf{28.5}} / \textcolor{red}{\textbf{0.861}} / \textcolor{red}{\textbf{0.172}} \\
        \bottomrule[0.1em]
    \end{tabular}
    }
    \vspace{-2mm}
    \label{tab:unseen_data}
\end{table}
To further validate the proposed method’s generalization capability on unseen data distribution, we evaluated its performance on unseen data across four representative image restoration tasks: low-light enhancement, deraining, denoising, and deblurring. The corresponding evaluation metrics are presented in Table \ref{tab:unseen_data}. Specifically, the low-light enhancement task was conducted using a combined dataset comprising MEF \citep{ma2015perceptual}, NPE \citep{wang2013naturalness}, and DICM \citep{lee2013contrast}; the deraining task employed the Practical \citep{yang2017deep} dataset; the denoising task utilized the SIDD \citep{abdelhamed2018high} dataset; and the deblurring task was assessed using the RealBlur \citep{rim2020real} dataset.

\subsection{PSNR results}
We provide the PSNR results of All-in-One-5 model as below. As shown in Tab.~\ref{tab:AllinOne5} of the paper, DRDD consistently achieves the best FID, LPIPS among all the methods, the best PSNR, SSIM among diffusion-based approaches, while remaining highly competitive on PSNR, SSIM with non-diffusion methods.
\begin{table}[h!]
\centering
\resizebox{\columnwidth}{!}{
\begin{tabular}{c|l|ccccc}
\hline
\multirow{2}{*}{\rotatebox{0}{\textbf{Table C.7}}}
& \textbf{Task} & \textbf{Low-Light} & \textbf{Deraining} & \textbf{Denoising} & \textbf{Deblurring} & \textbf{Dehazing} \\
& \textbf{Dataset} & \textbf{LOLv1} & \textbf{Rain100L} & \textbf{CBSD68} & \textbf{GoPro} & \textbf{SOTS} \\
\hline
\multirow{3}{*}{\rotatebox{0}{\footnotesize \textbf{Non-Diff.}}}
& DFPIR & \textcolor{red}{\textbf{23.80}} & 37.50 & \textcolor{blue}{31.26} & \textcolor{blue}{28.80} & \textcolor{red}{\textbf{31.24}} \\
& AdAIR & \textcolor{blue}{22.94} & \textcolor{blue}{37.85} & 31.29 & 28.11 & 29.93 \\
& VLU-NET & 22.25 & \textcolor{red}{\textbf{38.36}} & \textcolor{red}{\textbf{31.39}} & \textcolor{red}{\textbf{28.85}} & \textcolor{blue}{30.56} \\
\hline
\multirow{3}{*}{\rotatebox{0}{\textbf{Diff.}}} 
& DiffuIR & 19.33 & 34.88 & \textcolor{blue}{30.12} & 26.48 & \textcolor{blue}{30.27} \\
& DA-CLIP & \textcolor{blue}{20.12} & \textcolor{blue}{35.86} & 25.17 & \textcolor{blue}{27.34} & 26.88 \\
& \textbf{DRDD (Ours)} & \textcolor{red}{\textbf{23.00}} & \textcolor{red}{\textbf{36.86}} & \textcolor{red}{\textbf{31.47}} & \textcolor{red}{\textbf{29.08}} & \textcolor{red}{\textbf{30.56}} \\
\hline
\end{tabular}
}
\label{tab:psnr_comparison}
\vspace{-0.9em}
\end{table}

\subsection{Comparing data pruning results with more methods}
We provide comparison with AdaIR~\citep{cui2025adair} and RDDM~\citep{liu2024rddm} on pruned All-in-One-3 dataset, SSIM and LPIPS are reported.
\begin{table}[htbp]
    \centering
    \renewcommand\arraystretch{1.1}
    \resizebox{0.99\columnwidth}{!}{
        \begin{tabular}{l|cccc}
            \hline
            \multirow{2}{*}{\textbf{Table C.8}} & \textbf{25\%} & \textbf{50\%} & \textbf{75\%} & \textbf{100\%} \\
            \cline{2-5}
            & SSIM$\uparrow$ / LPIPS$\downarrow$ & SSIM$\uparrow$ / LPIPS$\downarrow$ & SSIM$\uparrow$ / LPIPS$\downarrow$ & SSIM$\uparrow$ / LPIPS$\downarrow$ \\
            \hline
            RDDM  & .929 / .058 & .935 / .056 & .941 / .049 & .942 / .047 \\
            AdAIR & .944 / .050 & .948 / .046 & .949 / .045 & .951 / .042 \\
            \hline
            \textbf{DRDD (Ours)} & \textbf{.947} / \textbf{.041} & \textbf{.949} / \textbf{.040} & \textbf{.950} / \textbf{.039} & \textbf{.951} / \textbf{.038} \\
            \hline
        \end{tabular}
    }
    \label{tab:data_scaling}
\end{table}

\subsection{Sensitivity Analysis of Noise Injection Level} 
Noise Injection level is relatively insensitive within the range of 0.8–1.3 (calculated via Eq.~\ref{eq:12} across different datasets), as validated by Fig.~\ref{fig:noise_scale} on the All-in-One-5 dataset and Table A5 (PSNR metric) on the Rain100H and Edges2Bags datasets.
Outside this 0.8–1.3 range, sensitivity increases.
If optimal performance is required, experimentation in 0.8–1.3 range is necessary.
 
\begin{table}[htbp]
    \centering
    \resizebox{\columnwidth}{!}{
    \begin{tabular}{l|c|c|c|c|c|c|c}
        \hline
        \textbf{Table C.9} & \textbf{Recommend $\sigma$}& \textbf{0.1} & \textbf{0.5} & \textbf{0.8} &\textbf{1.0} & \textbf{1.5} & \textbf{2.0}\\
        \hline
        Rain100H & 0.8-1.2 & 29.64 & 30.97 & 32.01 & 32.33 & 31.94 & 31.68 \\
        Edges2Bags & 0.8-1.1 & 16.34 & 19.22 & 20.05 & 19.81 & 18.76 & 17.92 \\
        \hline
    \end{tabular}
    }
    \label{tab:parameters}
\end{table}

\section{More Visual Comparisons}
\label{Appendix:d}
As shown in Fig.~\ref{fig:restoration_result} - Fig.\ref{fig:Irregular_inpainting_result}, we provide additional visual examples of DRDD on various image-to-image transformation tasks, including restoration and inpainting. These results serve as supplementary visualizations to those presented in the main paper.

\begin{figure*}[t]
    \centering
    \includegraphics[width=0.99\linewidth]{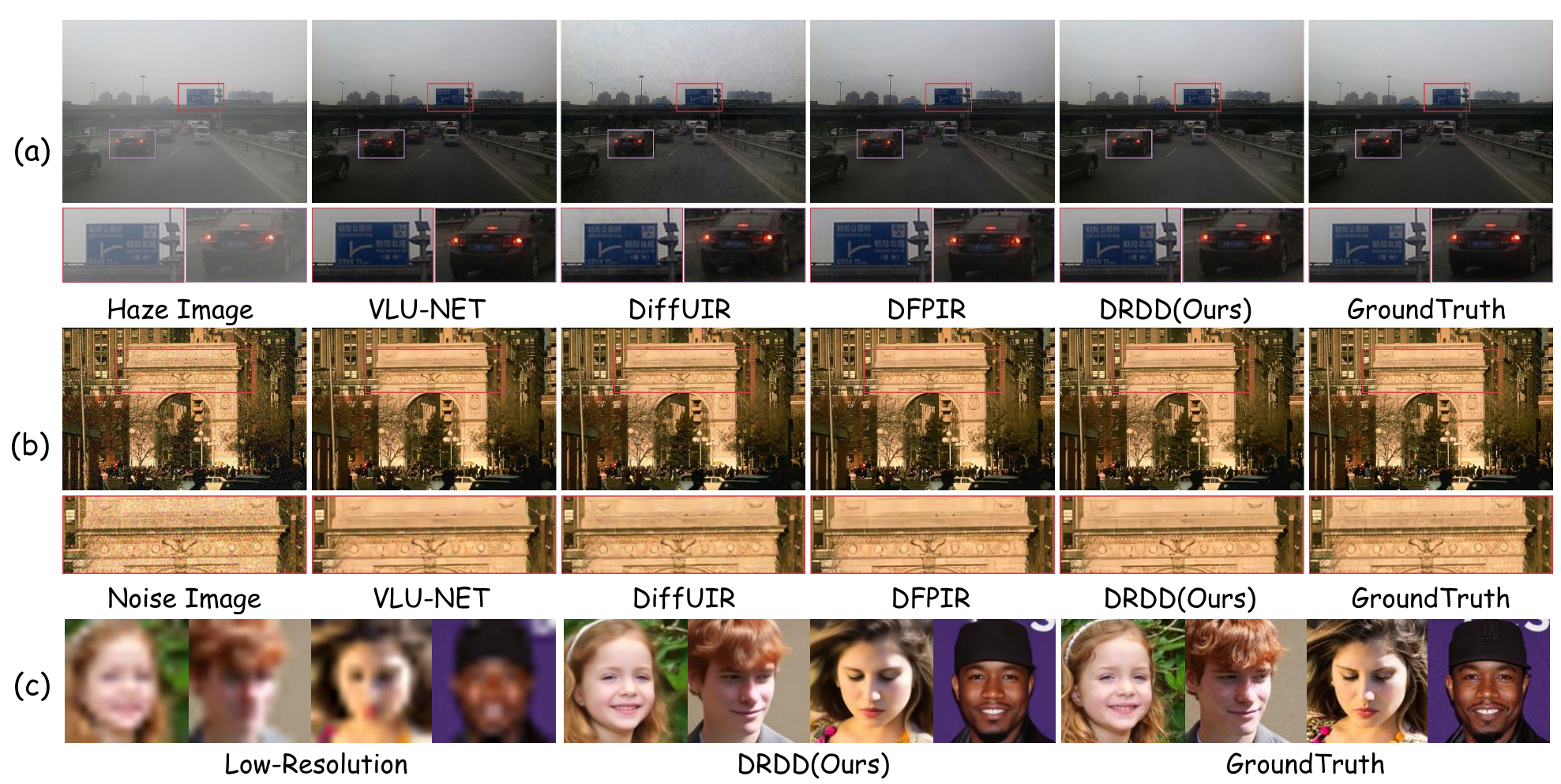}  
    \caption{\textit{Visual results of state-of-the-art methods and our proposed DRDD.} (a) Comparison of haze image restoration results on the SOTS dataset \citep{SOTS}. (b) Comparison of noise restoration results on the CBSD68 dataset~\citep{bsd400}. (c) Super-Resolution result in FFHQ~\citep{ffhq}. Zoom in for best view. } 
    \label{fig:restoration_result}  
\end{figure*}


\begin{figure*}[t]
    \centering
    \includegraphics[width=0.99\linewidth]{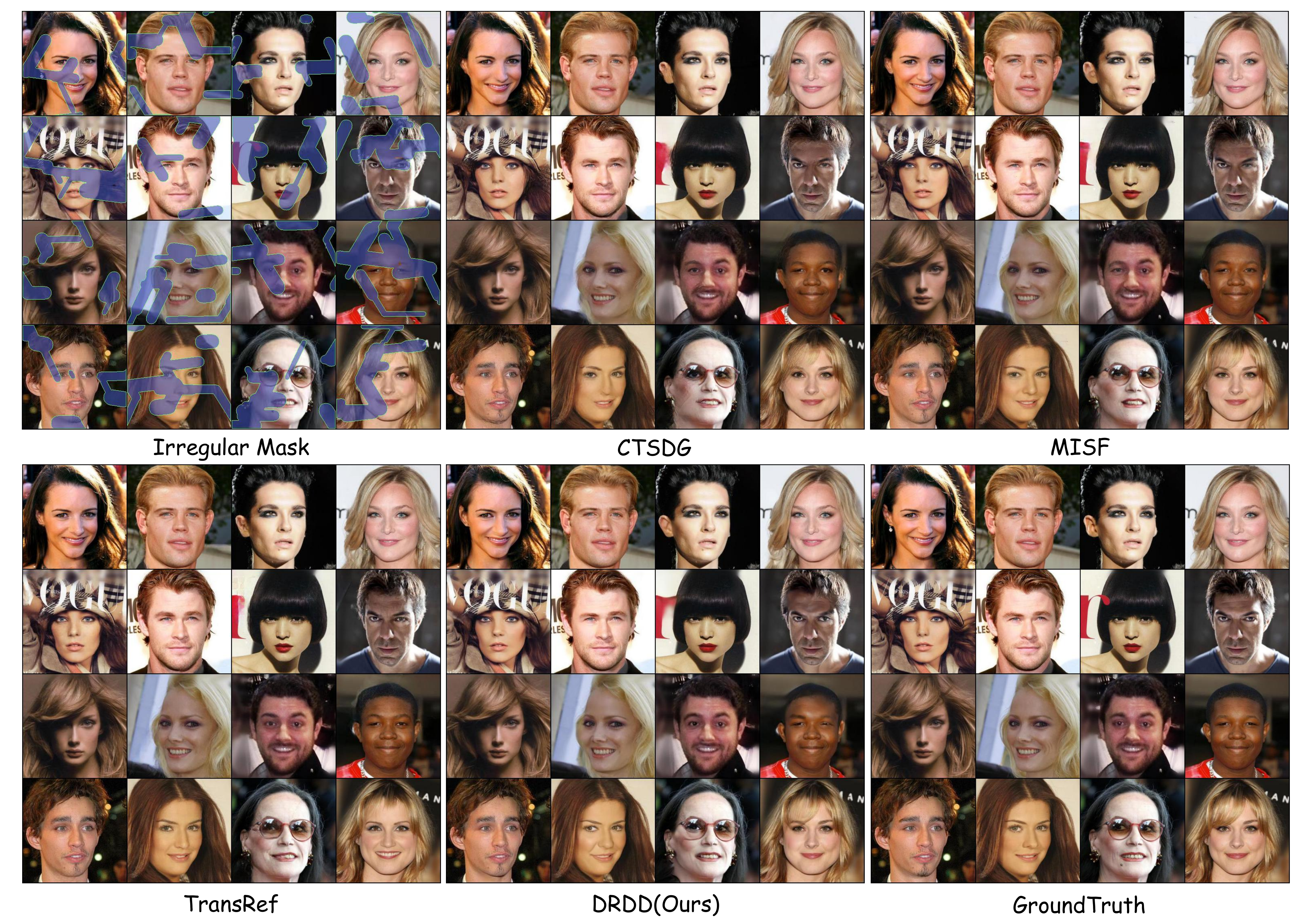}  
    \caption{\textit{Irregular Mask inpainting results of state-of-the-art methods and our proposed DRDD.} Zoom in for best view.} 
    \label{fig:Irregular_inpainting_result}  
\end{figure*}

\end{document}